\documentclass[11pt]{article}
\usepackage{fullpage}
\usepackage{float}

\usepackage[utf8]{inputenc} % allow utf-8 input
\usepackage[T1]{fontenc}    % use 8-bit T1 fonts
\usepackage{hyperref}       % hyperlinks
\usepackage{url}            % simple URL typesetting
\usepackage{booktabs}       % professional-quality tables
\usepackage{amsfonts}       % blackboard math symbols
\usepackage{nicefrac}       % compact symbols for 1/2, etc.
\usepackage{microtype}      % microtypography
\usepackage{xcolor}         % colors

\usepackage{breakcites}

\usepackage{amsthm,amsfonts,amssymb,amsmath}
\usepackage{xspace}
\usepackage{bbm}
\usepackage{graphicx}
\usepackage{bbm}

\usepackage{accents}

\usepackage{graphicx} % Allows including images
\usepackage{booktabs} % Allows the use of \toprule, \midrule and \bottomrule in tables
\usepackage{caption}
\usepackage{xcolor}
\usepackage{hyperref}

\usepackage{comment}
\usepackage{caption}
\usepackage{subcaption}
\usepackage{stfloats}

\usepackage{tikz}
\usepackage{appendix}

\usepackage{enumitem}

\usepackage[ruled,vlined]{algorithm2e}
\usepackage[noend]{algorithmic}

\usepackage[round]{natbib}

\usepackage{thmtools}
\usepackage{thm-restate}

\newtheorem{theorem}{Theorem}[section]
\newtheorem{lemma}{Lemma}[section]

\newtheorem{claim}{Claim}[section]

\newtheorem{corollary}{Corollary}[section]

\usepackage{xcolor}
\newcommand{\cj}[1]{\textcolor{blue}{[CJ: #1]}}

\newcommand{\cX}{\mathcal{X}}
\newcommand{\cY}{\mathcal{Y}}
\newcommand{\cP}{\mathcal{P}}
\newcommand{\bP}{\mathbb{P}}
\newcommand{\cA}{\mathcal{A}}

\newcommand{\pfair}{p^{\textsf{fair}}}

\newcommand{\ind}{\mathbbm{1}}

\newcommand{\cQ}{\mathcal{Q}}
\newcommand{\cF}{\mathcal{F}}

\newcommand{\cU}{\mathcal{U}}

\newcommand{\tP}{\tilde{\mathcal{P}}}
\newcommand{\E}{\mathop{\mathbb{E}}}
\newcommand{\hy}{\hat{y}}
\newcommand{\D}{\mathcal{D}}
\newcommand{\ha}{\hat{a}}

\newcommand{\R}{\mathbb{R}}

\newcommand{\Lagr}{\mathcal{L}}
\newcommand{\epi}{\text{epi }}

\newcommand{\norm}[1]{\left\lVert#1\right\rVert}
\newcommand{\inlineeqnum}{\refstepcounter{equation}~~\mbox{(\theequation)}}

\title{Distributionally Robust Data Join}
\author{Pranjal Awasthi\thanks{Google} \qquad Christopher Jung\thanks{Stanford University. Part of this work done while an intern at Google Research New York.} \qquad Jamie Morgenstern\thanks{University of Washington}}

\begin{document}

\maketitle

\begin{abstract}
Suppose we are given two datasets: a labeled dataset and unlabeled dataset which also has additional auxiliary features not present in the first dataset.
What is the most principled way to use these datasets together to construct a predictor? 

The answer should depend upon whether these datasets are generated by the same or different distributions over their mutual feature sets, and how similar the test distribution will be to either of those distributions. In many applications, the two datasets will likely follow different distributions, but both may be close to the test distribution. We introduce the problem of building a predictor which minimizes the maximum loss over all 
  probability distributions over the original features, auxiliary features, and binary labels, whose Wasserstein distance is $r_1$ away from the empirical distribution over the labeled dataset and $r_2$ away from that of the unlabeled dataset. This can be thought of as a generalization of distributionally robust optimization (DRO), which allows for two data sources, one of which is unlabeled and may contain auxiliary features. 
 
%We show how to approximately reformulate this problem with a convex optimization with theoretical approximation guarantees. And as an application, we show how our approach can help construct a distributionally robust fair model even when demographic information is not available in the original labeled data, but we have access to a separate unlabeled data with the demographic group information. 
\end{abstract}

\section{Introduction}
For a variety of prediction tasks, a number of sources of data may be available on which to train, each possibly following a distinct distribution. For example, health records might be available from at a number of geographically and demographically distinct hospitals. How should one combine these data sources  to build the best possible predictor? 

If the datasets $S_1, S_2$ follow different distributions $D_1, D_2$, the test distribution $D$ will necessarily differ from at least one. A refinement of our prior question is to ask for which test distributions, then, can training with $S_1, S_2$ 
give a good predictor?

More generally, very common issues of mismatch between training and test distributions (and uncertainty over which test distribution one might face)
has led to a great deal of interest in applying tools from distributionally robust optimization (DRO) to machine learning \citep{duchi2021learning, shafieezadeh2015distributionally, minimaxLearning, rahimian2019distributionally}. In contrast to classical 
 statistical learning theory, DRO picks a function $f$ whose maximum loss (over a set of distributions near $S$) is minimized. This set of potential test distributions, often referred to as the ambiguity or uncertainty set, captures the uncertainty over the test distribution, along with knowledge that the test distribution will be close to the training distribution.
 
 The ambiguity set is usually defined as a set of distributions with distance at most $r$ from  the empirical distribution over the training data: $B(\tP_S, r) = \left\{\cQ: D(\tP_S, \cQ) \le r \right\}$ where $\tP_S$ is the empirical distribution over training dataset $S$ and $D$ is some distance measure between two probability distributions. Then, DRO aims to find a model $\theta$ such that for some loss $\ell$, $\theta = \arg\min_{\theta} \sup_{\cQ \in B(\tP_S, r)} \E_{(x,y) \sim \cQ}[\ell(\theta, (x,y))]$ --- that is, minimize the loss over the worst case distribution in the ball of distributions $B(\tP_S, r)$. The larger $r$, the more distributions over which DRO hedges its performance, leading to a tension between performance (minimizing worst-case error) and robustness (over the set of distributions on which performance is measured).

In this work, we introduce a natural extension of distributionally robust learning, \emph{two anchor} distributionally robust learning, which we also refer to as the distributionally robust data join problem.  Two anchor distributionally robust learning has access to two sources of training data, the first source containing  labels, and the second source without labels but with  auxiliary features not present in the first source. The optimization is then over the set of distributions close to \emph{both} the labeled and auxiliary data distributions.

Formally, suppose one has  two training datasets. The first dataset $S_1$ consists of feature vectors $\cX = \R^{m_1}$ and binary prediction labels for some task $\cY=\{\pm 1\}$. The other dataset $S_2$ contains feature vectors $\cX$ and auxiliary features $\cA = \R^{m_2}$ but \emph{not} the labels. The goal is to find a model $\theta$ that hedges its performance against any distribution $\cQ$ over $(\cX, \cA, \cY)$ whose Wasserstein distance is $r_1$ away from the empirical distribution over $S_1$ and $r_2$ away from that of $S_2$. Note that our setting is a strict generalization of semi-supervised setting: for  $m_2=0$, there are no additional features in the second dataset, and $S_2$ is simply some additional unlabeled dataset. In contrast to pure semi-supervised settings, our method and setting both allow the learner to take advantage of the additional auxiliary features and to learn a model robust to additional distribution shift. We also emphasize that having the common features $x$ between $S_1$ and $S_2$ help learn about the relationship between the auxiliary features $a$ and the label $y$ indirectly. Consider the following example where we actually have one dataset that contains the feature vector, auxiliary features, and the label altogether $S^{\text{combined}}=\{(x_i,a_i,y_i)\}_{i=1}^n$. From this dataset, we may form $S_1=\{(x_i,a_i)\}_{i=1}^n$ and $S_2=\{(x_i, y_i)\}_{i=1}^n$ where for every point $(x_i,a_i)$ in $S_1$ and there's a corresponding $(x_i,y_i)$ such that they share the same feature. In fact, instantiating our framework with $r_1=0$ and $r_2=0$ corresponds exactly to performing empirical risk minimization over $S^{\text{combined}}$. In other words, the quality of how well feature vectors $x$'s match between $S_1$ and $S_2$ determine how well we may be able to learn the relationship between the auxiliary features $a$ and the label $y$.

In practice, it is quite common to have the datasets fragmented as our setting captures. For instance, suppose some dataset has been collected at a hospital in order to build a predictive model that is to be used at a nearby hospital. After collecting this data, some other research may find other features that could have been useful for the prediction task but unfortunately were not collected during the contruction of this dataset. Fortunately, another nearby hospital may have data that contains both the original features and the useful auxiliary features but does not have labels for this prediction task. Our data join approach allows to find a model that utilizes such auxiliary features and explicitly considers the distribution mismatch between the hospital where the model is deployed and the hospitals from which these two datasets have been collected.

Auxiliary features may be useful not only for improving accuracy of the model but for guaranteeing additional properties including notions of fairness. In Appendix~\ref{app:fairness}, we show that our distributionally robust data join problem encompasses a two-anchor distributionally robust learning instance where one can try to minimize not just the model's overall loss but also penalize the model for its difference in performance across demographic groups, even in situations where demographic information is present only in one dataset and the label is only present in the other dataset. This extension is motivated by designing equitable predictors (e.g., which equalize false positive rate over a collection of demographic groups) where one training set contains labels for the relevant task but no demographic information, and another training set contains demographic information but may not contain task labels. Such settings are quite common in practice, where demographic data is not collected for every dataset --- indeed, collection of demographic data is difficult to do well or sometimes even illegal \citep{awasthi2021evaluating, fremont2016race, weissman2011advancing, zhang2018assessing}.  

The contribution of our work can be summarized as follows: 
\begin{enumerate}
    \item New Problem Formulation of Distributionally Robust Data Join: we introduce and precisely formulate the distributionally robust data join problem in Section~\ref{sec:prelim} and exactly characterize its feasibility in Section~\ref{subsec:coupling-formulation}. 
    
    \item Application to Fairness: we further show how our original problem can be slightly modified to capture the problem of enforcing fairness when demographic group information is not available in the original labeled dataset (Appendix \ref{app:fairness}). 
    
    \item Tractable Reformulation with an Approximation Guarantee (Theorem \ref{thm:main-thm} in Section \ref{sec:tractable-optimization}): we show how to approximate the distributionally robust data join problem with two convex optimization problems with an approximation guarantee. 
    
    \item Experiments (Section \ref{sec:experiments}): we design and perform a synthetic experiment that shows how our distributionally robust data join method performs much better than the baselines. Additionally, we show some preliminary results on the experiments on a few real world datasets.  
\end{enumerate}

\subsection{Related Work}
\textbf{Distributionally Robust Optimization:}
Prior work has looked at many different ways to define the ambiguity set: characterizing the set  with moment and support information \citep{delage2010distributionally, goh2010distributionally, wiesemann2014distributionally}, or using various distance measures on probability space and defined the ambiguity set to be all the probability measures that are within certain distance $\epsilon$ of the empirical distribution: \cite{duchi2021learning} use f-divergence, \cite{hu2013kullback} the Kullback-Leibler divergence, \cite{erdougan2006ambiguous} the Prohorov metric, and \citet{shafieezadeh2015distributionally, blanchet2019quantifying, blanchet2019robust, esfahani2018data} the Wasserstein distance, \cite{hashimoto2018fairness} chi-square divergence, and so forth. Defining ambiguity sets with divergence measures suffers from the fact that they do not incorporate the underlying geometry between the points --- i.e. almost all divergence measures require the distribution in the ambiguity set to be absolutely continuous with respect to the anchor distribution. Therefore, because the distributions in the ambiguity set are simple re-weighting of the anchor distribution, divergence based ambiguity sets don't include distributions where the empirical distributions are perturbed a little bit and hence aren't robust to ``black swan'' outliers \citep{kuhn2019wasserstein}. By contrast, the Wasserstein distance allows one to take advantage of the natural geometry of the points (e.g. $L_p$ space). Furthermore, when we consider ambiguity sets defined by \emph{two} anchor distributions as we do in this work, the two empirical distributions that are the anchors of the ambiguity set are almost surely not continuous with respect to each other. For these reasons, we focus on the Wasserstein distance in this work.

Most relevant to our work from the distributionally robust optimization literature is \citet{shafieezadeh2015distributionally}. They show that regularizing the model parameter of the logistic regression has the effect of robustly hedging the model's performance against distributions whose distribution over just the covariates is slightly different than that of the empirical distribution over the training data. Distributionally robust logistic regression is a generalization of $p$-norm regularized logistic regression because it allows for not only distribution shift in the convariates but also the distribution shift over the labels. In a couple of real world datasets, they show that distributionally robust logistic regression seems to outperform regularized logistic regression by the same amount that regularized logistic regression outperforms vanilla logistic regression. Our work is a natural extension of this work in that we take additional unlabeled dataset with auxiliary features into account. However, we remark that our contributions go beyond the contributions of \cite{shafieezadeh2015distributionally}. In particular, reasoning about couplings between 3 distributions (labeled dataset, unlabeled dataset, and unknown target dataset) as shown later in Section~\ref{sec:dr-data-join} is \textit{a priori} not obvious and rather novel. Existing 2 distribution coupling approach used in \cite{shafieezadeh2015distributionally} (e.g., creating one coupling between labeled and unlabeled, and another between one of these and the test distribution) will not give empirically or theoretically good matchings between all three  distributions and will generally also not be computationally tractable in our case. We further discuss new technical difficulties that have to be overcome in order to solve our problem later in Section~\ref{sec:tractable-optimization} and Appendix~\ref{app:comparison-original-dro}. \cite{drofair} extend \cite{shafieezadeh2015distributionally} by adding a fairness regularization term, but the demographic information is available in the original labeled dataset in their setting unlike our setting.

\textbf{Semi-supervised Learning:} There have been significant advances in semi-supervised learning where the learner has access not only labeled data but also unlabeled data \citep{zhu2005semi, zhu2009introduction, chapelle2009semi}. While our setting is similar to 
semi-supervised settings, we capture a broader class of possible problems in two ways. First,
our approach allows the unlabeled dataset to have additional auxiliary features, and second, we explicitly take distribution shift into account.

\textbf{Imputation:} Numerous imputation methods for missing values in data exist, many of which have few or no theoretical guarantees \citep{donders2006gentle, royston2004multiple}. Many of these methods work best (or only have guarantees) when data values are missing at random. Our work, on the other hand,  assumes all prediction labels are missing from the second dataset and all auxiliary features are missing from the first dataset. Another related problem is the matrix factorization problem which is also referred to as matrix completion problem \citep{mnih2008probabilistic, koren2009matrix, candes2009exact}: here the goal is to find a low rank matrix that can well approximate the given data matrix with missing values. Our problem is different in that we don't make such structural assumption about the data matrix effectively being of low rank, but instead we assume all the auxiliary features are only available from a separate unlabeled dataset.

\textbf{Fairness:}
Many practical prediction tasks have disparate performance across demographic groups, and explicit demographic information may not be available in the original training data. Several lines of work aim to reduce the gap in performance of a predictor between groups even without group information for training.

\cite{hashimoto2018fairness} show that the chi-square divergence between the overall distribution and the distribution of any subgroup can be bounded by the size of the subgroup: e.g. for any sufficiently large subgroup, its divergence to the overall distribution cannot be too big. Therefore, by performing distributionally robust learning with ambiguity set defined by chi-square divergence, they are able to optimize for the worst-case risk over all possible sufficiently large subgroups even when the demographic information is not available. \cite{diana2020convergent} provide provably convergence oracle-efficient learning algorithms with the same kind of minimax fairness guarantees when the demographic group information is available.

One may naively think that given auxiliary demographic group information data, the most accurate imputation for the demographic group may be enough to not only estimate the unfairness of given predictor but also build a predictor with fairness guarantees. However, \cite{awasthi2021evaluating} show that due to different underlying base rates across groups, the Bayes optimal predictor for the demographic group information can result in maximally biased estimate of unfairness. \cite{diana2021multiaccurate} demonstrate that one can rely on a multi-accurate regressor, which was first introduced by \cite{kim2019multiaccuracy}, as opposed to a 0-1 classifier in order to estimate the unfairness without any bias and also build a fair classifier for downstream tasks. When only some data points are missing demographic information, \cite{jeong2021fairness} show how to bypass the need to explicitly impute the missing values and instead rely on some decision tree based approach in order to optimize a fairness-regularized objective function. \cite{kallus2021assessing}, given two separate datasets like in our setting, show how to construct confidence intervals for unfairness that is consistent with the given datasets via Fr{\'e}chet and Hoeffding inequalities; our work is different in that we allow a little bit of slack by forming a Wasserstein ball around both datasets and can actually construct a fair model as opposed to only measuring unfairness. 

\cite{celis2021fair} and \cite{celis2021fair2} have shown when the demographic group information is available but possibly noisy, stochastically and adversarially respectively, how to build a fair classifier.

\section{Preliminaries}
\label{sec:prelim}
\subsection{Notations}
We have two kinds of datasets, the auxiliary feature dataset and the prediction label dataset denoted in the following way: $S_A = \{(x^A_i, a^A_i)\}_{i=1}^{n_A},\quad S_P = \{(x^P_i, y^P_i)\}_{i=1}^{n_P}$ where the domain for feature vector $x$ is $\cX = \R^{m_1}$, domain for auxiliary features $a$ is $\cA = \R^{m_2}$, and the label space is $y \in \cY = \{\pm 1\}$. For any vector $v \in \R^{m}$ and $d_1, d_2 \in [m]$,  we write $v[d_1: d_2]$ to denote the coordinates from $d_1$ to $d_2$ of vector $v$ and $v[d]$ to denote the $d$th coordinate. We assume both $\cX$ and $\cA$ are compact and convex. For convenience, we write $S^{\cX}_A = \{x: (x,a) \in S_A \}, \quad S^{\cX}_P = \{x: (x,y) \in S_P \}$ to denote just the feature vectors of the dataset. 

Given any dataset $S=\{z_i\}_{i=1}^n$, we will write $\tP_S=\frac{1}{n} \sum_{i=1}^n \delta(z_i)$ to denote the empirical distribution over the dataset $S$ where $\delta$ is the Dirac delta funcion. We'll write $\bP_Z$ to denote the set of all probability distributions over $Z$. Similarly, we write $\bP_{(Z,Z')}$ to denote a set of all possible joint distributions over $Z$ and $Z'$. Also, given a joint distribution $\cP \in \bP_{(Z,Z')}$, we write $\cP_{Z}$ and $\cP_{Z'}$ to denote the marginal distribution over $Z$ and $Z'$ respectfully, meaning $\cP_Z(z) = \int \cP(z, dz')$ and $\cP_{Z'}(z') = \int \cP(dz, z')$. We extend the notation when the joint distribution is over more than two sets: e.g. $\cP_{z,z'}((z,z')) = \int \cP(z, z', dz'')$ where we have marginalized over $Z''$ for $\cP$ which is a joint distribution over $Z, Z', Z''$. 
%Sometimes, we abuse notation and write $\cP_1$ and $\cP_2$ to denote the marginal distribution over the first set, second set respectfully and $\cP_{(1,2)}$ to denote the marginal distribution over the first two sets.

We write the set of all possibly couplings between two distributions $\cP \in \bP_Z$ and  $\cP' \in \bP_{Z'}$ as $\Pi(\cP, \cP') = \left\{\pi \in \mathbb{P}_{(Z,Z')}: \pi_Z  = \cP, \pi_{Z'} = \cP' \right\}$. For a coupling between more than two distributions, we use the same convention and write $\Pi(\cP, \cP', \cP'')$ for instance.

Given any metric $d: Z \times Z \to \R$ and two probability distributions $\cP, \cP' \in \bP_Z$, we write the Wasserstein distance between them as $D_d(\cP, \cP') = \inf_{\pi \in \Pi(\cP, \cP')}\E_{(z,z') \sim \pi}[d(z, z')].$

Given some distribution $\cP \in \bP$ over some set $Z$, metric $d: Z \times Z \to \R$, a radius $r > 0$, we will write $B_{d}(\cP, r) = \left\{\cQ \in \bP: D_{d}(\cP, \cQ) \le r \right\}$ to denote the Wasserstein ball of radius $r$ around the given distribution $\cP$. When the metric is obvious from the context, we may simply write $B(\cP, r)$. 

In our case, the relevant metrics that are used to measure distance between points are
\begin{align*}
    d_\cX(x, x') &= || x- x'||_p, \quad d_A((x,a), (x',a')) = ||x - x'||_p + \kappa_A ||a - a' ||_{p'} \\
    d_P((x,y), (x',y')) &= ||x - x'||_p + \kappa_P |y-y'|
\end{align*}
where $|| v ||_p = \left(\sum_{d} |v[d]|^{p} \right)^\frac{1}{p}$ is some $p$-norm and $\kappa_A, \kappa_P \ge 0$ are the coefficients that control how much we care about the $||a-a'||_{p'}$ and $|y-y'|$. We'll write $|| v ||_{p,*} = \sup_{||v'||_p \le 1 } \langle v, v' \rangle$ to denote dual norm for $p$-norm. Also, for convenience, given any vector $v$, we'll write $\overline{v}_{p} = \frac{v}{||v||_{p}}$ and $\overline{v}_{p,*} = \frac{v}{||v||_{p,*}}$ to denote the normalized vectors. When it's clear from the context which norm is being used, we write $|| \cdot ||$, $|| \cdot ||_{*}$, $\overline{v}$, and $\overline{v}_{*}$. Now, we are ready to describe distributionally robust data join problem.

\subsection{Distributionally Robust Data Join}
\label{sec:dr-data-join}
We are given an auxiliary dataset $S_A$ and a prediction label dataset $S_P$. We are interested in a joint distribution $\cQ$ over $(x,a,y)$ such that
\begin{enumerate}
    \item its marginal distribution over $(x,a)$ is at most $r_A$ away from $\tP_{S_A}$ in Wasserstein distance: $\D_{d_A}(\cQ_{\cX, \cA}, \tP_{S_A})\le r_A$
    \item its marginal distribution over $(x,y)$ is at most $r_P$ away from $\tP_{S_P}$ in Wasserstein distance: $D_{d_P}(\cQ_{\cX, \cY}, \tP_{S_p}) \le r_P$
\end{enumerate}

Combining them together, the set of distributions we are interested in is 
\begin{align*}
    W(S_A, S_P, r_A, r_P) &=\{\cQ \in \bP_{(\cX, \cA, \cY)}:  \D_{d_A}(\cQ_{\cX, \cA}, \tP_{S_A})\le r_A,D_{d_P}(\cQ_{\cX, \cY}, \tP_{S_p}) \le r_P \}\\
    &= \{\cQ \in \bP_{(\cX, \cA, \cY)}: Q_{\cX, \cA} \in B_{d_A}(\tP_{S_A}, r_A), Q_{\cX,\cY} \in B_{d_P}(\tP_{S_P}, r_P) \}.
\end{align*}
Now, we consider some learning task where the performance is measured according to the worst case distribution in the above set of distributions. We want to find some model parameter $\theta$ such that its loss against the worst-case distribution among $W(S_A, S_P, r_A, r_P)$ is minimized:
\begin{align}
    \min_{\theta \in \Theta} \sup_{\cQ \in W(S_A, S_P, r_A, r_P)} \E_{(x,a,y) \sim \cQ}[\ell(\theta, (x,a,y))]. \label{eqn:disjoint-learning-orig}
\end{align}
where $\ell: \Theta \times (\cX \times \cA \times \cY) \to \R$ is a convex loss function evaluated at $\theta$. For the sake of concreteness, we focus on logistic loss\footnote{All our results still hold for any other convex loss with minimal modifications} $\ell(\theta, (x,a,y)) = \log(1+\exp(-y \langle \theta, (x,a) \rangle)).$

Also, we sometimes make use of the following functions $f(t)=\log(1+\exp(t))$ and $h(\theta, (x,a)) = f(-\langle \theta, (x,a) \rangle)$ instead of $\ell$, as it is more convenient due to not having to worry about $y$ in certain cases: $\ell(\theta, (x,a,+1)) = h(\theta, (x,a))$ and $\ell(\theta, (x,a,-1)) = h(-\theta, (x,a))$. We write the convex conjugate of $f$ as $f^*(b) = \sup_{x} \langle x^*, x \rangle - f(x)$, which in our case evaluates to $b  \log b  + (1-b) \log(1-b)$ when $b \in (0,1)$, 0 if $b=0$ or 1, and $\infty$ otherwise.

% \section{Examples}
% \label{sec:examples}
% \cj{TODO}
% \subsection{Semi-supervised Learning}
% \subsection{Fairness without Direct Access to Demographic Group Information}

\section{Tractable Reformulation}
\label{sec:tractable-optimization}
Let us give an overview of this section. Note that the optimization problem in \eqref{eqn:disjoint-learning-orig} is a saddle point problem. In Section \ref{subsec:coupling-formulation}, we first make the coupling in the optimal transport more explicit in the inner $\sup$ term. Then, by leveraging Kantorovich duality, we replace the $\sup$ term with its dual problem which is a minimization problem, thereby making the original saddle problem into minimization problem. However, the resulting dual problem has constraints that involve some supremum term, meaning it's an semi-infinite program (i.e. $\sup_{z \in Z} \text{constraint(z)} \le 0$ is equivalent to $\text{constraint(z)} \le 0, \forall z \in Z$). Finally, in Section \ref{subsec:removing-sup}, we show how each supremum term can be approximated by some other closed-form constraint. And we finally show that the resulting problem can be decomposed into two convex optimization problems and its optimal solution has additional approximation guarantee to the original optimal solution (Theorem~\ref{thm:main-thm}).

\subsection{Formulation through Coupling}\label{subsec:coupling-formulation}
We show how to rewrite the problem \eqref{eqn:disjoint-learning-orig} using the underlying coupling between the ``anchor'' distributions $(S_A, S_P)$ and $\cQ \in W(S_A, S_P, r_A, r_P)$. For simplicity, instead of $\pi\left((x^A_i, a^A_i), (x^P_j, y^P_j), (x,a,y)\right)$ which is a coupling between $\tP_{S_A}$, $\tP_{S_P}$, and some joint distribution $\cQ \in \bP_{\cX, \cA, \cY}$, we write $\pi^{y}_{i,j}(x, a)= \pi\left((x^A_i, a^A_i), (x^P_i, y^P_i), (x,a,y)\right)$. Then, since the ``anchor'' distributions $\tP_{S_A}$ and $\tP_{S_P}$ are discrete distributions, we can rewrite the problem \eqref{eqn:disjoint-learning-orig} as choosing $\theta \in \Theta$ that minimizes the following value: 
\begin{align}
     \sup_{\pi^{a,y}_{i,j}} \quad&  \sum_{i=1}^{n_A} \sum_{j=1}^{n_P} \sum_{y \in \cY}\int_{\cX, \cA} \ell(\theta, (x, a, y))  \pi^{y}_{i,j}(dx, da) \label{eqn:disjoint-learning-transport}\\
    \text{s.t.} \quad& \sum_{i=1}^{n_A}\sum_{j=1}^{n_P} \sum_{y \in \cY}\int_{\cX, \cA} d_A^{i}(x,a)  \pi^{y}_{i,j}(dx, da) \le r_A,\quad \sum_{i=1}^{n_A}\sum_{j=1}^{n_P} \sum_{y \in \cY} \int_{\cX, \cA} d_P^{j}(x,y)  \pi^{y}_{i,j}(dx, da) \le r_P \nonumber\\
    & \sum_{j=1}^{n_P} \sum_{y \in \cY} \int_{\cX, \cA} \pi^{y}_{i,j}(dx, da) = \frac{1}{n_A} \quad \forall i \in [n_A],\quad \sum_{i=1}^{n_A} \sum_{y \in \cY} \int_{\cX, \cA}  \pi^{y}_{i,j}(dx, da)  = \frac{1}{n_P} \quad \forall j \in [n_P] \nonumber
\end{align}
where $d^i_A(x,a) = d_A((x^A_i, a^A_i), (x,a))$ and $d^j_P(x,y) = d_P((x^P_j, y^P_j), (x,y))$. We defer intuitive explanations and derivation of this problem to Appendix \ref{app:coupling-formulation}. For any fixed parameter $\theta$, we'll denote the optimal value of the above problem \eqref{eqn:disjoint-learning-transport} as $p^*(\theta, r_A, r_P)$ and $p^*(r_A, r_P) = \inf_{\theta} p^*(\theta, r_A, r_P)$.

It can be shown that minimizing over the above supremum value in \eqref{eqn:disjoint-learning-orig} and the optimization problem \eqref{eqn:disjoint-learning-transport} are equivalent as shown in the following theorem. We also provide a tight characterization of the feasibility of \eqref{eqn:disjoint-learning-transport}. The proof of Theorem \ref{thm:coupling-formulation} and \ref{thm:feasibility} can be found in Appendix \ref{app:coupling-formulation}.

\begin{restatable}{theorem}{thmcouplingformulation}
\label{thm:coupling-formulation}
For any fixed $\theta \in \Theta$, $p^*(\theta, r_A, r_P) = \sup_{\cQ \in W(S_A, S_P, r_A, r_P)} \E_{(x,a,y) \sim Q}[\ell(\theta, (x,a,y))]$
\end{restatable}

\begin{restatable}{theorem}{thmfeasibility}
\label{thm:feasibility}
$D_{d_\cX}(\tP_{S^\cX_A}, \tP_{S^\cX_P}) \le r_A + r_P$, if and only if there exists a feasible solution for \eqref{eqn:disjoint-learning-transport}.
\end{restatable}

\subsection{Strong Duality}
\label{subsec:duality}
We claim that the following problem is the dual to problem \eqref{eqn:disjoint-learning-transport} and show that strong duality holds between them:
\begin{align}
    &\inf_{\substack{\alpha_A, \alpha_P, \\ \{\beta_{i}\}, \{\beta'_j\}}} \alpha_A r_A + \alpha_P r_P + \frac{1}{n_A}\sum_{i\in [n_a]} \beta_{i} + \frac{1}{n_P} \sum_{j \in [n_P]} \beta'\label{eqn:transport-dual}\\
    \text{s.t.} &\sup_{(x,a)} \left(\ell(\theta, (x, a, y)) - \alpha_A d^i_A(x,a)  -  \alpha_P d^j_P(x,y)\right) \le \beta_{i} + \beta'_j \quad \forall i \in [n_A], j \in [n_P], y \in \cY \nonumber
\nonumber
\end{align}
%where $d^i_A(x,a) = d^A((x^A_i,a^A_i), (x,a))$ and $d^j_P(x,y) = d^P((x^P_j,y^P_j), (x,y))$.

For fixed $\theta$, we'll write $d^*(\theta, r_A, r_P)$ to denote the optimal value for the above dual problem \eqref{eqn:transport-dual}. As in \cite{shafieezadeh2015distributionally}, strong duality directly follows from \cite{shapiro2001duality}, but to be self-contained, we include the proof in Appendix \ref{app:duality} which follows the same proof structure presented in \cite{villani2003topics}.

\begin{restatable}{theorem}{thmourduality}
\label{thm:our-duality}
If there exists a feasible solution for the primal problem \eqref{eqn:disjoint-learning-transport}, then we have that strong duality holds between the primal problem \eqref{eqn:disjoint-learning-transport} and its dual problem \eqref{eqn:transport-dual}: $p^*(\theta, r_A, r_P) = d^*(\theta, r_A, r_P)$ for fixed $\theta$. 
\end{restatable}

In other words, we have successfully transformed the saddle point problem~\eqref{eqn:disjoint-learning-orig} into a minimization problem over $\theta$ and the dual variables $\alpha_A, \alpha_P, \{\beta_i\}$ and $\{\beta'_j\}_j$: 
\begin{align}
    &\min_{\substack{\theta \in \Theta, \alpha_A, \alpha_P, \\ \{\beta_{i}\}, \{\beta'_j\}}} 
    \alpha_A r_A + \alpha_P r_P + \frac{1}{n_A}\sum_{i\in [n_a]} \beta_{i} + \frac{1}{n_P} \sum_{j \in [n_P]} \beta'_j \label{eqn:dual-prob}\\
    &\text{s.t.} \max_{y \in \{\pm 1\}}\sup_{(x,a)} \left(\ell(\theta, (x, a, y)) - \alpha_A d^i_A(x,a)  -  \alpha_P d^j_P(x,y)\right) \le \beta_{i} + \beta'_j \quad \forall i \in [n_A], j \in [n_P] \nonumber
\end{align}

\iffalse
\cj{DISCUSS: This part can be taken out. And instead discuss how the new problem we derive is a convex optimization problem.}
Now, we prove the convexity of the dual problem. It's clear that the objective is linear in terms of the dual variables. $\alpha_A, \alpha_P, \{\beta_i\}_{i \in [n_A]}, \{\beta'_j\}_{j \in [n_P]}$.
\begin{lemma}
$\alpha_A r_A + \alpha_P r_P + \frac{1}{n_A}\sum_{i\in [n_a]} \beta_{i} + \frac{1}{n_P} \sum_{j \in [n_P]} \beta'_j$ is linear in dual variables $\alpha_A, \alpha_P, \{\beta_i\}_i, \{\beta_j'\}_j$.
\end{lemma}
Because pointwise supremum preserves convexity, the constraint is also convex as well:
\begin{restatable}{lemma}{lemdualconvexconstraint}
\label{lem:dual-convex-constraint}
$\sup_{(x,a)} \ell(\theta, (x, a, y)) - \alpha_A d^i_A(x,a)  -  \alpha_P d^j_P(x,y) - \beta_{i} - \beta'_j$ is convex in paramter $\theta$ and dual variables $\alpha_A,\alpha_P, \beta_i, \beta'_j$.
\end{restatable}
\fi

\subsection{Replacing the $\sup$ Term}
\label{subsec:removing-sup}
Note that $\sup_{(x,a)}$ in the constraint makes it hard to actually compute the expression: it's neither concave or convex in terms of $(x,a)$ as it's the difference between convex functions $\ell(\theta, (x,a,y))$ and $\alpha_A d^i_A(x,a) + \alpha_P d^j_P(x,y)$. In that regard, we show how to approximate the $\sup$ term in the constraint of dual problem \eqref{eqn:transport-dual} with some closed form expression by extending the techniques used in \cite{shafieezadeh2015distributionally} who study when there's only one ``anchor'' point --- i.e. $\sup_{x} \ell(\theta, x) - \alpha d_\cX(x_i, x)$ as opposed to in our case with two anchor points. 

First, let's focus only on the terms that actually depend on $(x,a)$ and ignore our dependence on $y$ briefly:
\begin{align*}
    &\sup_{(x,a)} \ell(\theta, (x, a, y)) - \alpha_A d^i_A(x,a)  -  \alpha_P d^j_P(x,y)\\
    &=\kappa_P \alpha_P |y^P_j-y| + \left(\sup_{(x,a)} h(y \theta, (x, a)) - \alpha_A ||x^A_i - x||_p - \alpha_P ||x^P_j - x||_p + \alpha_A \kappa_A||a^A_i -a||_{p'}\right).
\end{align*}
We obtain an upper bound for the supremum term in the lemma below whose full proof can be found in Appendix~\ref{app:tractable-optimization}.
\begin{restatable}{theorem}{thmsupupper}
\label{thm:sup_upper}
Fix any $y \in \cY$ and $\theta$. Write $\theta_1 = \theta[1:m_1]$ and $\theta_2 = [m_1+1:m_1+m_2]$. Suppose $p \neq 1$ and $p \neq \infty$. If $||\theta_1||_{p,*} \le \alpha_A + \alpha_P$ and $||\theta_2||_{p',*} \le \kappa_A \alpha_A$, then
\begin{align*}
    &\sup_{(x,a)} h(y\theta, (x, a)) - \alpha_A ||x^A_i - x ||_{p} -  \alpha_P||x^P_j - x ||_{p} - \alpha_A \kappa_A ||a^A_i - a||_{p'} \\
    &\le f\Bigg(\Bigg(\frac{\min(\alpha_A,\alpha_P)||\theta_1||_{*}||x^A_i-x^P_j||}{\alpha_A + \alpha_P} + \frac{\langle y\theta_1, \alpha_A x^A_i + \alpha_P x^P_j\rangle}{\alpha_A + \alpha_P}\Bigg) + \langle y\theta_2, a^A_i\rangle\Bigg) - \min(\alpha_A, \alpha_P) || x^A_i - x^P_j||_{p}.
\end{align*}
Otherwise, $\sup_{(x,a)} h(y\theta, (x, a)) -\alpha_A||x -  x^A_i||_{p}  - \alpha_P||x-x^P_j||_{p} - \alpha_A \kappa_A ||a^A_i - a||_{p'}$ evaluates to $\infty$.
\end{restatable}
\begin{proof}[Proof Sketch]
    Similar to \cite{shafieezadeh2015distributionally}, we leverage convex conjugacy in order to re-express the $\sup$ term. However, because we have multiple anchor points, the re-expression results in an infimal convolution of \emph{two} linear functions with bounded norm constraints as opposed to the case of \cite{shafieezadeh2015distributionally} where they only have to handle a convex conjugate of a \emph{single} linear function with bounded norm constraint and hence find an exact closed form expression. Therefore, in Appendix~\ref{subsec:removing-sup} and~\ref{subsec:infimal-convolution}, we develop new techniques where we show (1) infimal convolution of linear functions with norm constraints is convex, (2) obtain a closed form solution of the infimal convolution at two extreme points, and (3) use linear interpolation of these extreme points to obtain an upper-bound, as a line segment of the two extreme points sits above the graph for convex functions.
\end{proof}

Equipped with the above upper bound on the supremum term, we can imagine trying to replace the supremum term with the above upper bound in order to get a feasible dual solution to the dual problem \eqref{eqn:dual-prob}. However, one may worry that there is a big gap between the original supremum term and our upperbound in Theorem~\ref{thm:sup_upper}. 

To this end, we further show that we can in fact approximate the supremum term with one more trick and hence obtain an approximate dual solution. Suppose we write $\hat{x}_{i,j} = \begin{cases}x^P_j \quad\text{if $\alpha_A < \alpha_P$}\\ x^A_i\end{cases}$ and  $\hat{\alpha} = \min(\alpha_A, \alpha_P)$. Note that by definition, the value measured at $(\hat{x}_{i,j}, a^A_i)$ is a lower bound on the supremum. In other words, we have
\begin{align*}
    &h(y\theta, (\hat{x}_{i,j}, a^A_i)) - \alpha_A ||x^A_i - \hat{x}_{i,j} ||_{p} -  \alpha_P||x^P_j - \hat{x}_{i,j} ||_{p} = f(\langle y\theta, (\hat{x}_{i,j}, a^A_i) \rangle) - \hat{\alpha} ||x^A_i - x^P_j||_{p}\\
    &\le \sup_{(x,a)} h(y\theta, (x, a)) - \alpha_A ||x^A_i - x ||_{p} -  \alpha_P||x^P_j - x ||_{p} - \alpha_A \kappa_A ||a^A_i - a||_{p'} \\
    &\le  f\Bigg(\Bigg(\frac{\min(\alpha_A,\alpha_P)||\theta_1||_{*}||x^A_i-x^P_j||}{\alpha_A + \alpha_P} + \frac{\langle y\theta_1, \alpha_A x^A_i + \alpha_P x^P_j\rangle}{\alpha_A + \alpha_P} + \langle y\theta_2, a^A_i\rangle \Bigg)- \hat{\alpha} || x^A_i - x^P_j||_{p}.
\end{align*}

Now, via H\"older's inequality, we can show the lower bound and the upper bound above on the supremum term are in fact very close, meaning by using either the upper bound or the lower bound, we can approximate the supremum very well. Here's a lemma that shows that the value evaluated at $(\hat{x}_{i,j}, a^A_i)$ is pretty close to the upper bound in Theorem~\ref{thm:sup_upper}:
\begin{restatable}{lemma}{lemholderlem}
\label{lem:holder-lem}
\begin{align*}
     f\Bigg(\Bigg(\frac{\min(\alpha_A,\alpha_P)||\theta_1||_{*}||x^A_i-x^P_j||}{\alpha_A + \alpha_P} + \frac{\langle y\theta_1, \alpha_A x^A_i + \alpha_P x^P_j\rangle}{\alpha_A + \alpha_P} + \langle y\theta_2, a^A_i\rangle \Bigg) -f(\langle y\theta, (\hat{x}_{i,j}, a^A_i)\rangle)
    &\le 2\hat{\alpha}||x^A_i - x^P_j||.
\end{align*}
\end{restatable}

In other words, replacing the original supremum constraint with a constraint evaluated at $(\hat{x}_{i,j}, a^A_i)$ will not incur too much additional error.
Finally, using the fact that $f(-t) = f(t) + t$ for logistic function $f$, we can bring back the terms that depend on $y$ and approximate the original supremum constraint in the following manner:
\begin{restatable}{corollary}{corsupapproxeverything}
\label{cor:sup-approx-everything}
\begin{align*}
&\left(\max_{y \in \{\pm 1\}}\sup_{(x,a)} \left(\ell(\theta, (x, a, y)) - \alpha_A d^i_A(x,a)  -  \alpha_P d^j_P(x,y)\right) \right) \\
&-\left(f(\langle y^P_j \theta, (\hat{x}_{i,j}, a^A_i) \rangle) + \max(y^P_j\langle\theta, (\hat{x}_{i,j}, a^A_i) \rangle-\alpha_P\kappa_P,0) -\hat{\alpha}||x^A_i -x^P_j|| \right)\le 2\hat{\alpha}||x^A_i - x^P_j||
\end{align*}
\end{restatable}

% Combining everything yields
% \begin{align*}
%     &\max_{y \in \{\pm 1\}} \sup_{(x,a)} \ell(\theta, (x, a, y)) - \alpha_A d^i_A(x,a)  -  \alpha_P d^j_P(x,y)\\
%     &=\max_{y \in \{\pm 1\}} \kappa_P \alpha_P |y^P_j-y| + \left(\sup_{(x,a)} h(y \theta, (x, a)) - \alpha_A ||x^A_i - x||_p - \alpha_P ||x^P_j - x||_p + \alpha_A \kappa_A||a^A_i -a||_{p'}\right)\\
%     &\le \max_{y \in \{\pm 1\}} \kappa_P \alpha_P |y^P_j-y| + f\Bigg(\Bigg(\frac{\min(\alpha_A,\alpha_P)||\theta_1||_{*}||x^A_i-x^P_j||}{\alpha_A + \alpha_P} + \frac{\langle y\theta_1, \alpha_A x^A_i + \alpha_P x^P_j\rangle}{\alpha_A + \alpha_P}\Bigg) + \langle y\theta_2, a^A_i\rangle\Bigg)\\
%     &- \min(\alpha_A, \alpha_P) || x^A_i - x^P_j||_{p} \quad\text{(Theorem~\ref{thm:sup_upper})}\\
%     &\le \max_{y \in \{\pm 1\}} \kappa_P \alpha_P |y^P_j-y| + f(\langle y \theta, (\hat{x}_{i,j}, a^A_{i,j})) + 2\min(\alpha_A, \alpha_P) || x^A_i - x^P_j||_{p} - \min(\alpha_A, \alpha_P) || x^A_i - x^P_j||_{p} \quad(\text{Lemma~\ref{lem:holder-lem}})\\
%     &\le f(\langle y^P_j \theta, (\hat{x}_{i,j}, a^A_{i,j})) + \min(\alpha_A, \alpha_P) || x^A_i - x^P_j||_{p} + \max(\langle y^P_j \theta, (\hat{x}_{i,j}, a^A_{i,j}), \kappa_P \alpha_P)
% \end{align*}
% where the last inequality follows from the fact that $f(-t) = f(t) + t$ for logistic function $f$. 
% Now, consider the following optimization problem where we replaced the original supremum constraint with the constraint evaluated at $(\hat{x}_{i,j}, a^A_i)$.
In other words, replacing the supremum constraint with the constraint evaluated at $(\hat{x}_{i,j}, a^A_i)$ and using the above trick to remove the max over $y$ will arrive at the following problem, for which we provide an approximation guarantee in Theorem~\ref{thm:main-thm}.
\begin{align}
    &\min_{\alpha_A, \alpha_P, \theta_1, \theta_2, \{\beta_i\}, \{\beta'_j\}} (\alpha_A r_A + \alpha_P r_P) +\frac{1}{n_A}\sum_{i\in [n_a]} \beta_{i} + \frac{1}{n_P} \sum_{j \in [n_P]} \beta'_j \label{eqn:convex-approx-prob}\\
    &\text{s.t.} f(y^P_j\langle\theta, (\hat{x}_{i,j}, a^A_i) \rangle) + \max(y^P_j\langle\theta, (\hat{x}_{i,j}, a^A_i) \rangle-\alpha_P\kappa_P,0) -\hat{\alpha}||x^A_i -x^P_j|| \le \beta_i + \beta'_j \quad \forall i \in [n_A], j \in [n_P] \nonumber \\
    &||\theta_1||_{*} \le \alpha_A + \alpha_P, ||\theta_2||_{*} \le \kappa_A \alpha_A. \nonumber
\end{align}
\begin{restatable}{theorem}{thmmainthm}
\label{thm:main-thm}
We can solve problem~\eqref{eqn:convex-approx-prob} by solving two convex optimization problems. And the optimal $\theta^*$ for the above problem ~\eqref{eqn:convex-approx-prob} is such that its objective value for the original problem \eqref{eqn:disjoint-learning-orig} is at most $ 2\hat{\alpha} \max_{i \in [n_A],j \in [n_P]} ||x^A_i - x^P_j||$ greater than the optimal solution: \begin{align*}
 &\sup_{\cQ \in W(S_A, S_P, r_A, r_P)} \E_{(x,a,y) \sim \cQ}[\ell(\theta^*, (x,a,y))] -2\hat{\alpha} \max_{i \in [n_A],j \in [n_P]} ||x^A_i - x^P_j|| \\
 \le &\min_{\theta \in \Theta} \sup_{\cQ \in W(S_A, S_P, r_A, r_P)} \E_{(x,a,y) \sim \cQ}[\ell(\theta, (x,a,y))].
\end{align*}
\end{restatable}

Just as in \cite{shafieezadeh2015distributionally}, two convex optimization problems that problem \eqref{eqn:convex-approx-prob} decomposes into can be solved by IOPT and YALMIP. In addition, we remark that $2\hat{\alpha} \max_{i \in [n_A], j \in [n_P]}||x^A_i - x^P_j||$ is a reasonable approximation guarantee because this value should be in the same order as $\alpha_A r_A + \alpha_P r_P$: recall that we have argued in Theorem \ref{thm:feasibility}, a feasible solution exists if and only if $D_{d_\cX}(\tP_{S^\cX_A}, \tP_{S^\cX_P}) \le r_A + r_P$. Additionally, the worst case pairwise distance can actually be improved with an additional assumption: since any underlying coupling for the Wasserstein distance most likely transports non-zero probability mass between only close points, we can imagine considering only the k-nearest-neighbors of each point as opposed to all possible pairs between two datasets, hence decreasing the approximation error to the maximal pairwise distance between some point and its k-nearest-neighbor. We make this point more formal in Appendix \ref{sec:opt}.

\section{Experiments}
\label{sec:experiments}
We now describe an experimental evaluation of our method on a synethetic dataset and real world datasets. In all our epxeriments, we use the approach discussed in Appendix~\ref{sec:opt} in which we make practically simplifying assumptions in order to solve the problem \eqref{eqn:convex-approx-prob} via projected gradient descent. We use $2$-norm throughout the experiments: i.e. $p,p'=2$.

\subsection{Synthetic Data}
\begin{table*}[t]
\centering
\begin{tabular}{|l|l|l|l|l|}
\hline
 & LR  & RLR & DRLR & DJ \\ \hline
Accuracy  & $0.4126\pm 0.1049$ & $0.5786\pm 0.3992$ & $0.9068 \pm 0.0076$ & $\textbf{0.9923} \pm 0.0057$ \\ 
\hline
\end{tabular}
\caption{\label{tab:synthetic-accuracy}Average accuracy of each method over 10 experiment runs and standard deviations for synthetic dataset with a distribution shift}
\end{table*}
We briefly discuss how we create the snythetic dataset. We want our synthetic data generation process to encompass the components that are unique to our robust data join setting --- namely, distribution shift and auxiliary unlabeled dataset that contains additional features that should help with the prediction task.

To that end, we discuss the data generation process at a high level here and more fully in Appendix \ref{app:experiments}. We have two groups such that the ideal hyperplane that distinguishes the positive and negative points is different for each group. We introduce distribution shift into the setting by having the original labeled training dataset consist mostly of points from the first group and the test dataset consist mostly from the second group. As for specific details of the data generation process that are important for our setting, we have one of the features to carry information regarding which group the point belongs to. 

As for the unlabeled dataset with auxiliary features, the points will mostly come from the second group, hence being closer to the test distribution. Furthermore, we include additional features that are present in the unlabeled dataset to be highly correlated with the true label, although this unlabeled dataset doesn't contain the true label of each point. 

Because we want our baselines that compare our distributionally robust data join approach (DJ) against to be in the same model class (i.e. logistic regression) as our method for fair comparison, we consider the following baselines:
\begin{enumerate}
    \item LR: Vanilla logistic regression trained on labeled dataset $S_P$
    \item RLR: Regularized logistic regression trainined on labeled dataset $S_P$
    \item DRLR: Distributionally robust logistic regression trainied on $S_P$
\end{enumerate}
The result of this experiment can be found in Table~\ref{tab:synthetic-accuracy}. There are few plausible reasons as to why our approach (DJ) does extremely well in this synthetic experiment. Our distributionally robust data join is definitely taking advantage of the proximity of unlabeled dataset to the test distribution in that the majority of points are both from the second group. Although regularized and distributionally robust logistic regression is trying to be robust against some form of distribution shift, the set of distributions they are hedging against may be too big as they are hedging against all distributions that are close to the empirical distribution over the labeled dataset. By contrast, the set of distributions that distributionally robust data join may be smaller because it's hedging against the set of distributions that are close to the labeled dataset \emph{and} the unlabeled dataset. Finally, auxiliary features in the unlabeled dataset are providing information very relevant for the prediction task.

\subsection{UCI Datasets}

\begin{table*}[t]
\centering
\centerline{
\begin{tabular}{|c|c|c|c|c|c|c|}
\hline
 & BC ($m_1 = 5$) & BC ($m_1 = 25$) & IO ($m_1 =4$) & IO ($m_1 =25$) & HD & 1vs8 \\ \hline
DJ  & $\textbf{0.9140}\pm 0.0368$ & $0.9281 \pm 0.0155$ & $\textbf{0.8208}\pm 0.0816$ & $\textbf{0.7896}\pm 0.04885$ & $\textbf{0.7495}\pm 0.0374$ & $\textbf{0.90841}\pm 0.0270$ \\ \hline
LR  & $0.9012 \pm 0.0294$ & $0.9140 \pm 0.0393$ &  $0.7764 \pm 0.1560$ &  $0.7868 \pm 0.0653$ & $0.7286 \pm 0.0504$ & $0.8729 \pm 0.0337$ \\ \hline
RLR & $0.9053\pm 0.0228$ & $\textbf{0.9287} \pm 0.0199$ & $0.7915 \pm 0.1417$ & $0.7868 \pm 0.0690$ & $0.7363 \pm 0.0565$ & $0.8953 \pm 0.0250$ \\ \hline
LRO & $0.8789 \pm 0.0318$ & $0.8789 \pm 0.0318$ & $0.7330 \pm 0.0788$ & $0.7330 \pm 0.0788$ & $0.6626 \pm 0.0569$ & $0.7766 \pm 0.0599$ \\ \hline
RLRO & $0.8953 \pm 0.0212$ & $0.8953 \pm 0.0212$ & $0.7377 \pm 0.0800$ & $0.7377 \pm 0.0800$ & $0.6714 \pm 0.0568$ & $0.8710 \pm 0.0450$ \\ \hline
\hline 
FULL & $0.9684 \pm 0.0143$ & $0.9684 \pm 0.0143$ & $0.8754 \pm 0.0764$ & $0.8754 \pm 0.0764$ & $0.8319 \pm 0.0311$ & $0.9495 \pm 0.0222$ \\ \hline
\end{tabular}
}
\caption{\label{tab:UCI-accuracy}Average accuracy of each method over 10 experiment runs and standard deviations for three UCI datasets} 
\end{table*}
% \begin{table*}[t]
% \centering
% \begin{tabular}{|c|c|c|c|c|}
% \hline
%  & BC ($m_1 = 5$) & IO ($m_1 =4$) & HD & 1vs8 \\ \hline
% DJ  & $\textbf{0.9199}\pm 0.0283$ & $\textbf{0.8226}\pm 0.0764$  & $\textbf{0.7495}\pm 0.0374$ & $\textbf{0.9206}\pm 0.0322$ \\ \hline
% LR  & $0.9012 \pm 0.0294$  &  $0.7764 \pm 0.1560$ & $0.7286 \pm 0.0504$ & $0.8729 \pm 0.0337$ \\ \hline
% RLR & $0.9053\pm 0.0228$  & $0.7915 \pm 0.1417$ & $0.7363 \pm 0.0565$ & $0.8953 \pm 0.0250$ \\ \hline
% LRO & $0.8789 \pm 0.0318$  & $0.7330 \pm 0.0788$  & $0.6626 \pm 0.0569$ & $0.7766 \pm 0.0599$ \\ \hline
% RLRO & $0.8953 \pm 0.0212$  & $0.7377 \pm 0.0800$ & $0.6714 \pm 0.0568$ & $0.8710 \pm 0.0450$ \\ \hline
% \hline 
% FULL & $0.9684 \pm 0.0143$ & $0.8754 \pm 0.0764$ & $0.8319 \pm 0.0311$ & $0.9495 \pm 0.0222$ \\ \hline
% \end{tabular}
% \caption{\label{tab:UCI-accuracy}Average accuracy of each method over 10 experiment runs and standard deviations for three UCI datasets} 
% \end{table*}

Here we discuss some experiments we have run and show that as a proof of concept, our distributionally robust data join framework has the potential to be practical empirically. However, we remark unlike in the synthetic data experiment, we do not introduce any distribution shift (i.e. training and test are iid samples from the same distribution) and also choose the additional features for the unlabeled dataset in an arbitrary way because of our lack of contextual expertise of the features in each dataset. Therefore, the gaps between our method and the baselines we consider are not as impressive as the performance gap we see in the synthetic experiments. 

 We use four UCI datasets for our real world dataset experiment: Breast Cancer dataset (BC), Ionosphere dataset (IO), Heart disease dataset (HD), and Handwritten Digits dataset with 1's and 8's (1vs8). We provide more details about these datasets in Appendix \ref{app:experiments}. For all these datasets, each experiment run consists of the following: (1) randomly divide the dataset into $S_{\text{train}} = \{(x_i, a_i, y_i)\}_{i=1}^{n_{\text{train}}}$ and $S_{\text{test}}$, (2) create the prediction label dataset and auxiliary dataset where $v$ data points belong to both datasets: $S_P = \{(x_i, y_i)\}_{i=1}^{n_P + v}$ and $S_A = \{(x_i, a_i)\}_{n_P+1}^{n_\text{train}}$.

We compare our method of joining $S_A$ and $S_P$, which we denote as DJ, to the following baselines: 
\begin{enumerate}
    \item LR: Logistic regression trained on $S_P$ 
    \item RLR: Regularized logistic regression on $S_P$
    \item LRO: Logistic regression on overlapped data $\{(x_i, a_i, y_i)\}_{i=n_P+1}^{n_P +v}$
    \item RLRO: Regularized logistic regression on overlapped data $\{(x_i, a_i, y_i)\}_{i=n_P+1}^{n_P +v}$.
    \item FULL: full training on $\{(x_i, a_i, y_i)\}_{i=1}^{n_{\text{train}}}$
\end{enumerate}
where FULL is simply to show the highest accuracy we could have achieved if the labeled dataset actually had the auxiliary features and the unlabeled dataset had the labels. The results of the experiment can be found in Table \ref{tab:UCI-accuracy}, and we include further details of the experiment in Appendix \ref{app:experiments}. Without any distribution shift, the distributionally robust data join method is solving a somewhat harder problem than the other baselines because of its hedging against other nearby distributions. Yet it can be seen that the use of the additional auxiliary features through our data join method helps achieve better accuracy than the baselines. 

\bibliographystyle{plainnat}
\bibliography{refs}
\onecolumn

\appendix
\section{Possible Negative Societal Impact and Limitations}
\label{app:limitation}
We do not foresee any direct negative societal impact of our work. However, just as other distributionally robust optimization methods, our robust guarantees may come at the price of achieving slightly worse accuracy. However, we note that this trade-off between more robustness and higher utility can be controlled by setting $r_A$ and $r_P$ appropriately. On a related note, another limitation of our approach is that it requires specifying $r_A$ and $r_P$; one needs to have some knowledge about how ``far'' the distributions (i.e. labeled dataset, unlabeled dataset with auxiliary features, and test distribution) may be, which is a limitation as in other methods that require setting some hyperparameters appropriately.

\section{Contributions Beyond \cite{shafieezadeh2015distributionally}}
\label{app:comparison-original-dro}
We describe our method's novelty and additional technical difficulties that we overcome specifically with repsect to \cite{shafieezadeh2015distributionally}. 

\paragraph{Use of the unlabeled dataset with auxiliary features}
The most obvious novelty is the use of the unlabeled dataset with auxiliary features. As discussed in the introduction, the benefit is two-fold in that it allows one to take advantage of the availability of the unlabeled dataset and also auxiliary features present in unlabeled dataset and the test distribution.

\paragraph{``Tripling'' between labeled dataset, unlabeled dataset with auxiliary features, and the test distribution} We emphasize that extending the result of \cite{shafieezadeh2015distributionally} to create a new anchor for the unlabeled dataset $S_P$ is not immediate. One naive way to reformulate the problem \eqref{eqn:disjoint-learning-orig} is to consider three couplings: a coupling $\pi_A$ between $\tP_{S_A}$ and the test distribution $\cQ_1$ and a coupling $\pi_P$ between $\tP_{S_P}$ and test distribution $\cQ_2$, and finally a coupling $\pi_{\text{test}}$ between $\cQ_1$ and $\cQ_2$. Because the test distribution that we consider has to be one single test distribution, we have to ensure that the Wasserstein distance between $\cQ_1$ and $\cQ_2$ be 0 with this formulation. This final constraint adds much complication that causes the problem to be not computationally tractable. In other words, our way of consider a coupling simultaneously over $\tP_{S_A}$, $\tP_{S_P}$, and the test distribution $\cQ$ is novel.

\paragraph{Infimal Convolution}
The use of ``tripling'' as described in Section~\ref{subsec:coupling-formulation} yields a dual that looks very similar to that of \cite{shafieezadeh2015distributionally}. However, our problem has two anchor points as opposed to a single anchor point due to our use of two datasets as opposed to one. When there is only a single anchor point as in \cite{shafieezadeh2015distributionally}, looking at the convex conjugates of the terms in the supremum yields $\infty$ when the dual norm of $\theta$ is bigger than $\alpha$ or evaluates to the loss evaluated at the anchor point --- see their Lemma 1. In our case, due to having two anchor points, we get an infimal convolution of \emph{two} linear functions with bounded norm constraints, which doesn't have a closed form as before. Hence, we additionally show that infimal convolution is convex, calculate the closed form of the infimal convolution at the extreme points and bound the overall supremum term via linear interpolation at these points. We further show in Theorem~\ref{thm:main-thm} that our upper-bound has an additive approximation error guarantee and how with the structural assumption about the optimal transport, we can expect the additive approximation error to be small in Appendix~\ref{sec:opt}. 

\paragraph{Gradient-based Optimization Approach}
\cite{shafieezadeh2015distributionally} resort to YALMIP and IPOPT, which are modeling language for MATLAB and the state-of-the-art nonlinear programming solver respectively in order to solve their problem. Our method can be similarly solved 
via such methods. We additionally show that we can use a gradient-based optimization approach to solve the problem given the structural assumption about $M$.

\section{Missing Details from Section \ref{sec:tractable-optimization}}
\label{app:tractable-optimization}
\subsection{Missing Details from Section \ref{subsec:coupling-formulation}}
\label{app:coupling-formulation}
Note that marginalizing $\pi$ over $i \in [n_A]$ yields a coupling $\pi_{S_P, (\cX,\cA,\cY)}$ between $\tP_{S_P}$ and $Q$, and marginalizing over $j \in [n_P]$ yields a coupling $\pi_{S_A, (\cX,\cA,\cY)}$ between $\tP_{S_A}$ and $\cQ$. Similarly, $\pi$'s marginal distribution over $(\cX,\cA,\cY)$, $S_A$ and $S_P$ is exactly $\cQ$, $\tP_{S_A}$, $\tP_{S_P}$ respectively.
\begin{align*}
    \pi_{S_P, (\cX,\cA,\cY)}((x^P_j, y^P_j), (x,a,y)) &= \sum_{i=1}^{n_A} \pi\left((x^A_i, a^A_i), (x^P_j, y^P_j), (x,a,y)\right) \\
    \pi_{S_A, (\cX,\cA,\cY)}((x^A_i, a^A_i), (x,a,y)) &= \sum_{j=1}^{n_P} \pi\left((x^A_i, a^A_i), (x^P_j, y^P_j), (x,a,y)\right) \\
    \pi_{S_A}(x^A_i, a^A_i) &= \sum_{j=1}^{n_P} \int \pi\left((x^A_i, a^A_i), (x^P_j, y^P_j), (dx,da,dy)\right) = \frac{1}{n_A} \\
    \pi_{S_P}(x^P_j, y^P_j) &= \sum_{i=1}^{n_A} \int \pi\left((x^A_i, a^A_i), (x^P_j, y^P_j), (dx,da,dy)\right) = \frac{1}{n_P}\\
    \cQ = \pi_{(\cX,\cA,\cY)}(x,a,y) &= \sum_{i=1}^{n_A}\sum_{j=1}^{n_P} \pi\left((x^A_i, a^A_i), (x^P_j, y^P_j), (x,a,y)\right) \\
\end{align*}

Using the above notations, we can re-write the constraint in $W(S_A, S_P, r_A, r_P)$ where $\pi$'s marginal distribution over $(\cX,\cA)$ must be at most $r_A$ away from $\tP_{S_A}$ in Wasserstein distance as follows:
\begin{align*}
    &\E_{(x^A_i,a^A_i), (x,a,y)) \sim \pi_{S_A, (\cX,\cA,\cY)}}\left[d_{A}((x^A_i,a^A_i), (x,a))\right] \\
    &= \sum_{i=1}^{n_A} \int d_{A}((x^A_i,a^A_i), (x,a)) \pi_{S_A, (\cX,\cA,\cY)}(x^A_i, a^A_i, (dx,da,dy)) \\ 
    &= \sum_{i=1}^{n_A} \sum_{j=1}^{n_P} \int d_{A}((x^A_i,a^A_i), (x,a))  \pi\left((x^A_i, a^A_i), (x^P_j, y^P_j), (dx,da,dy)\right) \le r_A.
\end{align*}

Similarly, we can write the other constraint that $\pi$'s marginal distribution over $(\cX, \cY)$ must be at most $r_P$ away from $\tP_{S_P}$ as 
\begin{align*}
    &\E_{(x^P_j,y^P_j), (x,a,y)) \sim \pi_{S_P, (\cX,\cA,\cY)}}\left[d_{P}((x^P_j,y^P_j), (x,y))\right] \\
    &= \sum_{i=1}^{n_A} \sum_{j=1}^{n_P} \int d_{P}((x^P_j,a^P_j), (x,a))  \pi\left((x^A_i, a^A_i), (x^P_j, y^P_j), (dx,da,dy)\right) \le r_P.
\end{align*}

Lastly, the constraint that in order $\pi$ to be a valid coupling, its marginal distribution over $S_A$ and $S_P$ should be exactly $\frac{1}{n_A}$ and $\frac{1}{n_P}$ over its support is equivalent to
\begin{align*}
    \sum_{j=1}^{n_P}\sum_{a \in A} \sum_{y \in \cY} \int \pi\left((x^A_i, a^A_i), (x^P_j, y^P_j), (dx,da,dy)\right)  = \frac{1}{n_A} \quad \forall i \in [n_A]\\
    \sum_{i=1}^{n_A}\sum_{a \in A} \sum_{y \in \cY} \int \pi\left((x^A_i, a^A_i), (x^P_j, y^P_j), (dx,da,dy)\right)  = \frac{1}{n_P} \quad \forall j \in [n_P].
\end{align*}

\thmcouplingformulation*
\begin{proof}
It's clear that for any feasible solution $\pi$ for \eqref{eqn:disjoint-learning-transport}, we must have that \[\pi_{(\cX,\cA,\cY)} \in W(S_A, S_P, r_A, r_P)
\]
as we have a coupling $\pi_{S_A, (\cX, \cA, \cY)}$ between $\tP_{S_A}$ and $\pi_{(\cX,\cA,\cY)}$ such that \[
 \E_{(x^A_i,a^A_i), (x,a,y)) \sim \pi_{S_A, (\cX,\cA,\cY)}}\left[d_{A}((x^A_i,a^A_i), (x,a))\right]  \le r_A
\]
and
\[
    \E_{(x^P_j,y^P_j), (x,a,y)) \sim \pi_{S_P, (\cX,\cA,\cY)}}\left[d_{P}((x^P_j,y^P_j), (x,y))\right] \le r_P.
\]

Also, for any $Q \in W(S_A, S_P, r_A, r_P)$, let's write the optimal transport between $\tP_{S_A}$ and $Q$ as $\pi^*_{S_A, (\cX, \cA, \cY)}$ and the optimal transport between $\tP_{S_P}$ and $Q$ as $\pi^*_{S_P, (\cX, \cA, \cY)}$. Then consider the following coupling between $\tP_{S_A}, \tP_{S_P},$ and $Q$:
\[
    \pi((x^A_i a^A_i), (x^P_j, y^P_j), (x,a,y)) = \pi^*_{S_A, (\cX, \cA, \cY)}((x^A_i a^A_i), (x,a,y)) \cdot \pi^*_{S_P, (\cX, \cA, \cY)}((x^P_j, y^P_j), (x,a,y)). 
\]
which is a product of $\pi^*_{S_A, (\cX, \cA, \cY)}$ and $\pi^*_{S_P, (\cX, \cA, \cY)}$. This $\pi$ is clearly a feasible solution for \eqref{eqn:disjoint-learning-transport}. $\pi_{S_A, (\cX, \cA, \cY)} = \pi^*_{S_A, (\cX, \cA, \cY)}$ which witnesses that its Wasserstein distance to $\tP_{S_A}$ is at most $r_A$, and the same argument applies for $\tP_{S_P}$. Also, its marginal distribution over $S_A$ and $S_P$ will be exactly $\tP_{S_A}$ and $\tP_{S_P}$ respectively because  both $\pi^*_{S_A, (\cX, \cA, \cY)}$ and $\pi^*_{S_P, (\cX, \cA, \cY)}$ is a valid coupling for $\tP_{S_A}$ and $\tP_{S_P}$ respectively.

Therefore, we must have \[p^*(\theta, r_a, r_p) = \sup_{Q \in W(S_A, S_P, r_A, r_P)} \E_{(x,a,y) \sim Q}[\ell(\theta, (x,a,y))]\]
\end{proof}

\subsection{Feasibility of Problem \eqref{eqn:disjoint-learning-transport}}
Here we focus on the feasibility of problem \eqref{eqn:disjoint-learning-transport}: more specifically, how big $r_A$ and $r_P$ needs to be in order for $W(S_A, S_P, r_A, r_P)$ to be a non-empty set. 

\thmfeasibility*
\begin{proof}
\textbf{$(\Rightarrow)$ direction:}
    Suppose $\pi^* = \arg\min_{\pi \in \Pi(\tP_a^{X}, \tP_a^{X})}\E_{\pi}[d(x,x'))]$ is the coupling between $\tP_{S^X_a}$ and $\tP_{S^X_p}$ from that results in the Wasserstein distance $D_{d_\cX}(\tP_{S^\cX_A}, \tP_{S^\cX_P}) = \E_{(x^A_i, x^P_j)\sim \pi^*}[d_X(x^A_i, x^P_j)]$. 
    
    For every $i \in [n_A]$ and $j \in [n_P]$, define 
    \begin{align*}
        x^*_{i,j} &= x^A_i - \frac{r_A}{r_A + r_P}(x^A_i - x^P_j)\\
        &=x^P_j +\frac{r_p}{r_A + r_P}(x^A_i - x^P_j)
    \end{align*}
    which is essentially a weighted average of $x^A_i$ and $x^P_j$.
    
    Note that we have
    \begin{align*}
        ||x^*_{i,j} - x^A_i|| &= ||x^A_i - \frac{r_A}{r_A + r_P}(x^A_i - x^P_j) - x^A_i||\\
        &=\frac{r_A}{r_A+r_P}||x^A_i - x^P_j||
    \end{align*}
    \begin{align*}
        ||x^*_{i,j} - x^P_j|| &= || x^P_j +\frac{r_P}{r_A + r_P}(x^A_i - x^P_j) - x^P_i||\\
        &=\frac{r_P}{r_A+r_P}||x^A_i - x^P_j||
    \end{align*}
    
    Then, construct $\pi^{a,y}_{i,j}$ as follows:
    \[
        \pi^{y^P_j}_{i,j}(x^*_{i,j}, a^A_i) = \pi^*(x^A_i, x^P_j)
     \]
    and 0 otherwise: in other words, for each $(i,j)$, there's a point mass of $\pi^*(x^A_i, x^P_j)$ at $x^*_{i,j}, a^A_{i,j}$ with $y=y^P_j$. We now show that the constructed coupling $\pi^{y}_{i,j}$ is a feasible solution for \eqref{eqn:disjoint-learning-transport}.

First, note that we can prove that its marginal distribution transport cost is bounded by $r_A$ and $r_P$. In the case of $S_A$, we have
\begin{align*}
    \sum_{i=1}^{n_A}\sum_{j=1}^{n_P} \sum_{y \in \cY}\int_{\cX,\cA} d^i_A(x,a)  \pi^{y}_{i,j}(dx, da) &=\sum_{i=1}^{n_A}\sum_{j=1}^{n_p} d^i_A(x^*_{i,j})  \pi^{y^p_j}_{i,j}(x^*_{i,j}, a^A_i)\\
    &=\sum_{i=1}^{n_A}\sum_{j=1}^{n_p} \left(|| x^*_{i,j} - x^A_{i}|| + \kappa_A ||a^A_i - a^A_i ||\right) \pi^*(x^A_i, x^P_j)\\
    &=\frac{r_A}{r_A + r_P}\sum_{i=1}^{n_A}\sum_{j=1}^{n_p} || x^P_{j} - x^A_{i}||  \pi^*(x^A_i, x^P_j)\\
    &\le r_A.
\end{align*}

For $S_P$, we can similarly show
\begin{align*}
    \sum_{i=1}^{n_A}\sum_{j=1}^{n_P} \sum_{y \in \cY}\int_{\cX, \cA} d^j_P(x,a)  \pi^{y}_{i,j}(dx, da) &=\sum_{i=1}^{n_A}\sum_{j=1}^{n_p} d^j_P(x^*_{i,j})  \pi^{ y^p_j}_{i,j}(x^*_{i,j}, a^A_i)\\
    &=\sum_{i=1}^{n_A}\sum_{j=1}^{n_p} \left(|| x^*_{i,j} - x^P_{j}|| + \kappa_P|y^P_j - y^P_j|\right) \pi^*(x^A_i, x^P_j)\\
    &=\frac{r_P}{r_A + r_P}\sum_{i=1}^{n_A}\sum_{j=1}^{n_p} || x^P_{j} - x^A_{i}||  \pi^*(x^A_i, x^P_j)\\
    &\le r_P.
\end{align*}

Finally, the constructed $\pi^{a,y}_{i,j}$ is a valid coupling:
\begin{align*}
    \sum_{j=1}^{n_P} \sum_{y \in \cY} \int_{\cX,\cA} \pi^{y}_{i,j}(dx, da) = \sum_{j=1}^{n_P} \pi^*(x^A_i, x^P_j)= \frac{1}{n_A} \quad \forall i \in [n_A]\\
    \sum_{i=1}^{n_A} \sum_{y \in \cY} \int_{\cX,\cA} \pi^{y}_{i,j}(dx, da)  =  \sum_{i=1}^{n_A}\pi^*(x^A_i, x^P_j) = \frac{1}{n_P} \quad \forall j \in [n_P],
\end{align*}
as $\pi^*$ was a valid coupling between $\tP_{S_A}$ and $\tP_{S_P}$.

\textbf{$(\Leftarrow)$ direction:}
We'll use $\pi^{a}_{i,j}$ to denote the feasible solution to \eqref{eqn:disjoint-learning-transport}. Now, construct a coupling $\pi$ such that the expected transport cost between $\tP_{S^\cX_A}$ and $\tP_{S^\cX_P}$ under $\pi$ is at most $r_A + r_P$, meaning the Wasserstein distance is at most $\min(r_a, r_p)$.

Construct the coupling $\pi$ between $\tP_{S^\cX_A}$ and $\tP_{S^\cX_P}$ as 
\begin{align*}
    \pi(x^A_i, x^P_j) = \sum_{y \in \cY}\int_{\cX,\cA} \pi^{y}_{i,j}(dx, da).
\end{align*}

It's easy to see that $\pi$ is a valid coupling as 
\begin{align*}
    \sum_{i=1}^{n_A} \pi(x^A_i, x^P_j) = \sum_{i=1}^{n_A} \sum_{y \in \cY}\int_{\cX,\cA} \pi^{y}_{i,j}(dx, da) = \frac{1}{n_P}\\ \sum_{j=1}^{n_P} \pi(x^A_i, x^P_j)  = \sum_{j=1}^{n_P} \sum_{y \in \cY}\int_{\cX,\cA} \pi^{y}_{i,j}(dx,da) = \frac{1}{n_A}
\end{align*}
for each $i \in [n_A]$ and $j \in [n_P]$.

Finally, due to its feasibility, we get
\begin{align}
    \sum_{i=1}^{n_A} \sum_{j=1}^{n_P} \sum_{y \in \cY} \int_{\cX,\cA} d^i_A(x,y)  \pi^{y}_{i,j}(dx,da) \le r_A \nonumber \\
    \sum_{i=1}^{n_A} \sum_{j=1}^{n_P} \sum_{y \in \cY} \int \left(||x^A_i -x  || + \kappa_a||a^A_i - a|| \right) \pi^{y}_{i,j}(dx, da)  \le r_A \nonumber \\
    \sum_{i=1}^{n_A} \sum_{j=1}^{n_P} \sum_{y \in \cY} \int_{\cX,\cA} ||x^A_i - x  ||  \pi^{y}_{i,j}(dx, da)  \le r_A \label{eqn:feasibility-a-constr}
\end{align}
Similarly, we get
\begin{align}
    \sum_{i=1}^{n_A} \sum_{j=1}^{n_P} \sum_{y \in \cY} \int_{\cX,\cA} ||x^P_j - x  ||  \pi^{y}_{i,j}(dx,da)  \le r_P  \label{eqn:feasibility-p-constr}
\end{align}
By adding \eqref{eqn:feasibility-a-constr} and \eqref{eqn:feasibility-p-constr}, we get
\begin{align*}
    \sum_{i=1}^{n_A} \sum_{j=1}^{n_P} \sum_{y \in \cY} \int_{\cX, \cA} \left(||x^A_i -x  || + ||x^P_j -x  || \right)\pi^{y}_{i,j}(dx, da) \le r_A + r_P \\
    \sum_{i=1}^{n_A} \sum_{j=1}^{n_P} \sum_{y \in \cY} \int_{\cX, \cA} ||x^A_i -x^P_j  || \pi^{y}_{i,j}(dx, da) \le r_A+ r_P \\
    \sum_{i=1}^{n_A} \sum_{j=1}^{n_P} ||x^A_i -x^P_j  || \pi(x^A_i, x^P_j) \le r_A + r_P.
\end{align*}
The second line follows from the triangle inequality $||x^A_i -x^P_j  || \le ||x^A_i -x  || + ||x^P_j -x  ||$. Therefore, we have $D_{d_\cX}(\tP_{S^\cX_A}, \tP_{S^\cX_P}) \le r_A + r_P$.
\end{proof}

\subsection{Missing Details from Section \ref{subsec:duality}}
\label{app:duality}
\thmourduality*
\begin{proof}
 For clarity, we assume in the proof that $\cX$ and $\cA$ is compact, but for more interested readers, we refer to the strong duality proof in Theorem 1.3 of \cite{villani2003topics} to see how to remove the compactness assumption on $\cX$ and $\cA$.

This theorem essentially follows from Fenchel-Rokafellar Duality which is formally stated later in the proof. Before applying the duality theorem, it is instructive to take a look at the corresponding Lagrangian for \eqref{eqn:disjoint-learning-transport}:
% \begin{align*}
%     \Lagr(\alpha_A, \alpha_P, \{\beta_{i,j}\}_{i,j}) &= \sup_{\{Q_{i,j}\}_{i,j}} \frac{1}{n_p n_a}\sum_{i=1}^{n_a}\sum_{j=1}^{n_p} \int \ell_p(\theta, (x, a, y))  Q_{i,j}(dx, da, dy))\\
%     &+ \alpha_A \left(r_a - \frac{1}{n_p n_a}\sum_{i=1}^{n_a}\sum_{j=1}^{n_p} \int d^i_A(x,a)  Q_{i,j}(dx, da, dy) \right) \\
%     &+ \alpha_P \left(r_p - \frac{1}{n_p n_a}\sum_{i=1}^{n_a}\sum_{j=1}^{n_p} \int d^j_P(x,y)  Q_{i,j}(dx, da, dy) \right)\\
%     &+ \sum_{i\in [n_a], j\in[n_p]} \beta_{i,j} \left(1- \int Q_{i,j}(dx,da,dy)\right).
% \end{align*}
\begin{align*}
    \Lagr(\pi, \alpha_A, \alpha_P, \{\beta_{i,j}\}_{i,j}) &= \sum_{i=1}^{n_A}\sum_{j=1}^{n_P} \sum_{y \in \cY}\int_{\cX, \cA} \ell(\theta, (x, a, y))  \pi^{y}_{i,j}(dx, da))\\
    &+ \alpha_A \left(r_A - \sum_{i=1}^{n_A}\sum_{j=1}^{n_P} \sum_{y \in \cY}\int_{\cX,\cA} d^i_A(x,a)  \pi^{y}_{i,j}(dx, da)  \right) \\
    &+ \alpha_P \left(r_P - \sum_{i=1}^{n_A}\sum_{j=1}^{n_P}  \sum_{y \in \cY} \int d^j_P(x,y)  \pi^{y}_{i,j}(dx, da)  \right)\\
    &+ \sum_{i=1}^{n_A} \beta_{i} \left(\frac{1}{n_A}- \sum_{j=1}^{n_P} \sum_{y \in Y} \int_{\cX,\cA} \pi^{y}_{i,j}(dx, da) \right)\\
    &+ \sum_{j=1}^{n_P} \beta'_{j} \left(\frac{1}{n_P} - \sum_{i=1}^{n_A} \sum_{y \in Y} \int_{\cX,\cA} \pi^{y}_{i,j}(dx, da) \right).
\end{align*}

Rearranging the terms yields
\begin{align*}
    &\Lagr(\pi, \alpha_A, \alpha_P, \{\beta_{i}\}, \{\beta'_{j}\}\})\\
    &= \sum_{i=1}^{n_A}\sum_{j=1}^{n_P} \sum_{y \in \cY}\int_{\cX,\cA} \left(\ell(\theta, (x, a, y)) - \alpha_A d^i_A(x,a) -  \alpha_P d^j_P(x,y) - \beta_i  - \beta'_j \right) \pi^{y}_{i,j}(dx, da)\\
    &+ \alpha_A r_A + \alpha_P r_P + \frac{1}{n_A} \sum_{i=1}^{n_A} \beta_i + \frac{1}{n_P} \sum_{j=1}^{n_P} \beta'_j. 
\end{align*}

Note that the optimal primal value can be written in terms of its Lagrangian: 
\[
p^*(\theta, r_A, r_P) = \sup_{\pi} \inf_{\alpha_A, \alpha_P, \{\beta_{i}\}, \{\beta'_j\}} \Lagr(\pi, \alpha_A, \alpha_P, \{\beta_{i}\}, \{\beta'_{j}\}).
\]

For notational economy, we'll write
\[
\psi(\alpha_A, \alpha_P, \{\beta_{i}\}, \{\beta'_{j}\}, x, a, y) = -\alpha_A d^i_A(x,a) - \alpha_P d^j_P(x,y) - \beta_i - \beta'_j
\]

Now, we state Fenchel's duality theorem:
\begin{theorem}[Fenchel-Rokafellar Duality]\label{thm:fenchel-duality}
Let $E$ be a normed vector space, and let $f,g: E \to \mathbb{R} \cup \{+\infty\}$ be two convex functions. Assume there exists $z_0 \in E$ such that $f(z_0) < \infty$ and $g(z_0) < \infty$, and $f$ and $g$ are continuous at $z_0$. Then, 
\[
    \inf_{E}(g + f) = \sup_{z^* \in E^*}(-g^*(-z^*) - f^*(z^*))
\]
\end{theorem}

By Riesz's theorem, we have that the dual space of the Radon measure $\pi^{y}_{i,j}$ is the continuous bounded functions which we denote as $u(i,j,x,a,y)$. In our case, define 
\begin{align*}
    f(u) = \begin{cases}
         0 &\quad\text{if $u(i,j, x,a,y) + \ell(\theta, (x, a, y)) \le 0$ for all $i \in [n_A]$ and $j \in [n_P]$}\\
         \infty &\quad\text{otherwise}
    \end{cases}
\end{align*}
\begin{align*}
    &g(u) \\
    &=\begin{cases}
         \left(\alpha_A r_A + \alpha_P r_P + \frac{1}{n_A} \sum_{i=1}^{n_A} \beta_i + \frac{1}{n_P} \sum_{j=1}^{n_P} \beta'_j \right) &\quad\substack{\text{if $u(i,j, x,a,y) =  \psi(\alpha_A, \alpha_P, \{\beta_{i}\}, \{\beta'_j\}, x, a, y)$} \\ \text{for some $\alpha_A, \alpha_P, \{\beta_{i}\}, \{\beta'_j\}$}}\\
         \infty &\quad\text{otherwise}
    \end{cases}
\end{align*}

Note that both $f$ and $g$ are convex:
\begin{enumerate}
    \item $f$ is convex\\
    Consider any $u,v$ such that $f(u) < \infty$ and $f(v) < \infty$, then $u(i,j,x,a,y) \le -\ell(\theta, x,a,y)$ and $v(i,j,x,a,y) \le -\ell(\theta, x,a,y)$. Then, because $t  u(i,j,x,a,y) + (1-t)  v(i,j,x,a,y) \le -\ell(\theta, x,a,y)$, we have 
    \[t f(u) + (1-t) f(v) = 0 = f(t (u) + (1-t) v).\] 
    If either $f(u) = \infty$ or $f(v) = \infty$, then 
    \[f(t (u) + (1-t) v) \le t f(u) + (1-t) f(v).\]
    
    \item $g$ is convex\\
    Suppose $u, v$ is such that $g(u) < \infty$ and $g(v) < \infty$ and $g(u) = \alpha^u_A r_A + \alpha^u_P r_P + \frac{1}{n_A}\sum_{i\in [n_A]} \beta^u_{i} + \frac{1}{n_P}\sum_{j\in [n_P]} {\beta'}^u_{j}$ and $g(v) = \alpha^v_A r_A + \alpha^v_P r_P + \frac{1}{n_A}\sum_{i\in [n_A]} \beta^v_{i}+ \frac{1}{n_P}\sum_{j \in [n_P]} {\beta'}^v_j$. Then, we have $t g(u) + (1-t) g(v) = g(tu + (1-t)v)$. If $g(u) = \infty$ or $g(v) = \infty$, it's easy to see that $g(tu + (1-t)v) \le \infty$ as well.
    
\end{enumerate}

Note that 
\begin{align*}
    &\inf_{u}(f(u) + g(u)) \\
    &= \inf_{\ell(\theta, (x, a, y)) - \alpha_A d^i_A(x,a)  -  \alpha_P d^j_P(x,y) - \beta_{i} - \beta'_j \le 0} \left(\alpha_A r_A + \alpha_P r_P + \frac{1}{n_A}\sum_{i\in [n_A]} \beta_{i} + \frac{1}{n_P}\sum_{j\in [n_P]} \beta'_{j}\right) \\
    &= d^*(r_A, r_P)
\end{align*}

We derive their convex conjugates: 
\begin{align*}
    f^*(\{\pi^{a,y}_{i,j}\}) &= \sup_{u + \ell \le 0}  \sum_{i=1}^{n_A}\sum_{j=1}^{n_P} \sum_{y \in \cY} \int_{\cX,\cA} u(i,j, x, a, y)\pi^{y}_{i,j}(dx, da) \\
    &= -\sum_{i=1}^{n_A}\sum_{j=1}^{n_P} \sum_{y \in \cY}\int_{\cX,\cA} \ell(\theta, (x, a, y)) \pi^{y}_{i,j}(dx, da)
\end{align*}
\begin{align*}
    &g^*(\{\pi^{a,y}_{i,j}\}) \\
    &= \sup_{u(i,j, x,a,y) =  \psi(\alpha_A, \alpha_P, \{\beta_{i}\}, \{\beta'_{j}\}, x, a, y)}
    \sum_{i=1}^{n_A}\sum_{j=1}^{n_P} \sum_{y \in \cY}\int_{\cX,\cA} u(i,j,x,a,y) \pi^{y}_{i,j}(dx, da) \\
    &-\left(\alpha_A r_A + \alpha_P r_P + \frac{1}{n_A}\sum_{i\in [n_A]} \beta_{i} + \frac{1}{n_P}\sum_{j \in [n_P]} \beta'_j\right) \\
    &= \sup_{\alpha_A, \alpha_P, \{\beta_{i}\}, \{\beta'_j\}_j}\sum_{i=1}^{n_A}\sum_{j=1}^{n_P} \sum_{y \in \cY}\left(\int_{\cX,\cA} - \alpha_A d^i_A(x,a)  -  \alpha_P d^j_P(x,y) - \beta_{i} - \beta'_{j} \right) \pi^{y}_{i,j}(dx, da)\\
    &-\left(\alpha_A r_A + \alpha_P r_P + \frac{1}{n_A}\sum_{i\in [n_A]} \beta_{i} + \frac{1}{n_P}\sum_{j \in [n_P]} \beta'_j \right).
\end{align*}

Also, note that 
\begin{align*}
    &\sup_{\pi} (-g^*(-\pi) - f^*(\pi)) \\
    &=\sup_{\pi} \inf_{\alpha_A, \alpha_P, \{\beta_{i,j}\}_{i,j}}\sum_{i=1}^{n_A}\sum_{j=1}^{n_P} \sum_{y \in \cY}\left(\int_{\cX,\cA} - \alpha_A d^i_A(x,a)  -  \alpha_P d^j_P(x,y) - \beta_{i} - \beta'_j \right) \pi^{y}_{i,j}(dx, da)\\
    &+\left(\alpha_A r_A + \alpha_P r_P + \frac{1}{n_A}\sum_{i\in [n_A]} \beta_{i} + \frac{1}{n_P}\sum_{j \in [n_P]} \beta'_{j} \right) \\
    &+\sum_{i=1}^{n_A}\sum_{j=1}^{n_P} \sum_{y \in \cY}\int_{\cX,\cA} \ell(\theta, (x, a, y)) \pi^{y}_{i,j}(dx, da) \\
    &= \sup_{\pi} \inf_{\alpha_A, \alpha_P, \{\beta_{i}\}, \{\beta'_{j}\}}\Lagr(\pi, \alpha_A, \alpha_P, \{\beta_{i}\}, \{\beta'_j\})\\
    &=p^*(\theta, r_A, r_P).
\end{align*}

Therefore, by Theorem \ref{thm:fenchel-duality}, we see that $p^*(\theta, r_A, r_P) = d^*(\theta, r_A, r_P)$.
\end{proof}

\subsection{Missing Details from Section \ref{subsec:removing-sup}}
We rearrange some terms of $R$ and use convex conjugate of $h$ to represent the supremum term with what is known as an infimal convolution:
\begin{restatable}{lemma}{lemconjugateinfimalrepr}
\label{lem:conjugate-infimal-repr}
Fix any $\theta$, $(x^A_i, a^A_i, x^P_j)$, and $(\alpha_A, \alpha_P, \kappa_A)$. If $||\theta[1:m_1]||_{p,*} > \alpha_A + \alpha_P$ or $||\theta[m_1+1: m_1 + m_2]||_{p',*} > \kappa_A \alpha_A$, then $\sup_{(x,a)} h(\theta, (x, a)) - \alpha_A ||x^A_i - x ||_{p} -  \alpha_P||x^P_j - x ||_{p} - \alpha_A \kappa_A ||a^A_i - a||_{p'} = \infty$.
Otherwise, we have 
\begin{align*}
    &\sup_{(x,a)} h(\theta, (x, a)) - \alpha_A ||x^A_i - x ||_{p} -  \alpha_P||x^P_j - x ||_{p} - \alpha_A \kappa_A ||a^A_i - a||_{p'}\\
    &=\sup_{b \in [0,1]} - f^*(b) + (g^i_1 \square g^j_2)(b\theta[1:m_1]) + \langle b\theta[m_1+1:m_1+m_2], a^A_i\rangle
\end{align*}
where \begin{align*}
    g^i_1(\theta) = \begin{cases}
         \langle \theta, x^A_i \rangle &\quad\text{if $||\theta||_{p,*} \le \alpha_A$}\\
         \infty &\quad\text{otherwise}
    \end{cases}\quad 
    g^j_2(\theta) = \begin{cases}
         \langle \theta, x^P_j \rangle &\quad\text{if $||\theta||_{p,*} \le \alpha_P$}\\
         \infty &\quad\text{otherwise}
    \end{cases}
\end{align*}
and
$(g^i_1 \square g^j_2)(\theta) = \inf_{\theta_1 + \theta_2 = \theta} g^i_1(\theta_1) + g^j_2(\theta_2)$ is the infimal convolution of $g^i_1$ and $g^j_2$. 
\end{restatable}
\begin{proof}
Noting that $h$ is convex and thus $h$ is equal to its biconjugate $h^{**}$, we have
\begin{align*}
    &\sup_{(x,a)} h(\theta, (x, a)) - \alpha_A ||x^A_i - x ||_{p} -  \alpha_P||x^P_j - x ||_{p} - \alpha_A \kappa_A ||a^A_i - a||_{p'}\\
    &=\sup_{(x,a)} \sup_{b \in [0,1]} \langle b\theta, (x,a) \rangle  - f^*(b) - \alpha_A ||x^A_i - x||_{p}  -  \alpha_P ||x^P_j - x||_{p} - \alpha_A \kappa_A ||a^A_i - a||_{p'}\\
    &= \sup_{b \in [0,1]} \sup_{(x,a)} \langle b\theta, (x,a) \rangle  - f^*(b) - \sup_{||q_1||_{p,*} \le \alpha_A}\langle q_1, x^A_i - x \rangle -  \sup_{||q_2||_{p,*} \le \alpha_P}\langle q_2, x^P_j - x\rangle - \sup_{||q_3||_{p',*} \le \alpha_A \kappa_A}\langle q_3, a^A_i - a \rangle \\ 
    &=  \sup_{b \in [0,1]} \sup_{(x,a)} \langle b\theta, (x,a) \rangle  - f^*(b) \\
    &- \sup_{||q_1||_{p,*} \le \alpha_A}\langle (q_1, 0), (x^A_i,a) - (x,a) \rangle -  \sup_{||q_2||_{p,*} \le \alpha_P}\langle (q_2, 0), (x^P_j,a) - (x,a)\rangle - \sup_{||q_3||_{p',*} \le \alpha_A \kappa_A}\langle (0,q_3), (x,a^A_i) - (x,a) \rangle\\
    &=  \sup_{b \in [0,1]} \sup_{(x,a)} \inf_{\substack{||q_1||_{p,*}\le \alpha_A,\\ ||q_2||_{p,*}\le \alpha_P, \\||q_3||_{p',*} \le \alpha_A \kappa_A}} \langle b\theta, (x,a) \rangle  - f^*(b)\\
    &- \langle (q_1, 0), (x^A_i,a) - (x,a) \rangle - \langle (q_2, 0), (x^P_j,a) - (x,a)\rangle - \langle (0,q_3), (x,a^A_i) - (x,a) \rangle\\
    &= \sup_{b \in [0,1]} \sup_{(x,a)}  \inf_{\substack{||q_1||_{p,*}\le \alpha_A,\\ ||q_2||_{p,*}\le \alpha_P, \\||q_3||_{p',*} \le \alpha_A \kappa_A}} \langle b\theta + (q_1, 0) + (q_2,0) + (0,q_3), (x,a) \rangle  - f^*(b) -\langle q_1, x^A_i \rangle - \langle q_2, x^P_j \rangle - \langle q_3, a^A_i \rangle.
\end{align*}   

We can swap the order of $\inf$ and $\sup$ due to proposition 5.5.4 of \cite{bertsekas2009convex}. 
\begin{align*}
    &= \sup_{b \in [0,1]}  \inf_{\substack{||q_1||_{p,*}\le \alpha_A,\\ ||q_2||_{p,*}\le \alpha_P, \\||q_3||_{p',*} \le \alpha_A \kappa_A }} \sup_{(x,a)}  \langle b\theta + (q_1, 0) + (q_2,0) + (0,q_3), (x,a) \rangle  - f^*(b) - \langle q_1, x^A_i \rangle - \langle q_2, x^P_j\rangle - \langle q_3, a^A_i \rangle.
\end{align*}
Note that unless $b\theta + (q_1, 0) + (q_2,0) + (0,q_3) = 0$, $(x,a)$ can be chosen arbitrarily big. Also, if $\theta + (q_1, 0) + (q_2, 0) + (0,q_3)\neq 0$, then $b$ can be chosen to be 1. Therefore, if there doesn't exist $(q_1, q_2, q_3)$ such that $\theta + (q_1, 0) + (q_2,0) + (0,q_3) = 0$, everything evaluates to $\infty$. 
In other words, the expression evaluates to $\infty$ unless both of the following conditions are true:
\begin{enumerate}
    \item $||\theta[1:m_1]||_{p,*} \le \alpha_A + \alpha_P$
    \item $||\theta[m_1 + 1: m_1 + m_2]||_{p',*} \le \kappa_A \alpha_A$
\end{enumerate}
as $q_1=\frac{-\alpha_A}{||\theta[1:m_1]||}\theta[1:m_1]$, $q_2 = \frac{-\alpha_P}{||\theta[1:m_1]||}\theta[1:m_1]$, and $q_3 = \frac{\kappa_A \alpha_A}{||\theta[m_1+1:m_1+m_2]||_{*}}\theta[m_1+1:m_1+m_2]$ is one such triplet that satisfy $\theta + (q_1, 0) + (q_2,0) + (0,q_3) = 0$.

Now, suppose $\theta$ satisfies the above conditions as we know it evaluates to $\infty$ otherwise. Then, we get
\begin{align*}
    &= \sup_{b \in [0,1]} - f^*(b) \\
    &+ \inf_{\substack{||q_1||_{p,*}\le \alpha_A,\\ ||q_2||_{p,*}\le \alpha_P, \\||-b\theta[m_1 +1: m_1 + m_2]||_{p',*} \le \alpha_A \kappa_A}} \begin{cases}
          - \langle q_1, x^A_i \rangle - \langle q_2, x^P_j \rangle + \langle b\theta[m_1 +1: m_1 + m_2], a^A_i\rangle &\quad\text{if $b\theta[1:m_1] + q_1 + q_2 = 0$} \\
         \infty &\quad\text{otherwise}
    \end{cases}\\
    &=\sup_{b \in [0,1]} - f^*(b) + \inf_{\substack{||q_1||_{p,*}\le \alpha_A,\\ ||q_2||_{p,*}\le \alpha_P}} \begin{cases}
          \langle q_1, x^A_i \rangle + \langle q_2, x^P_j \rangle &\quad\text{if $q_1 + q_2 = b\theta[1:m_1]$}\\
         \infty &\quad\text{otherwise}
    \end{cases} +\langle b\theta[m_1+1:m_1+m_2], a^A_i\rangle \\
    &=\sup_{b \in [0,1]} - f^*(b) + (g^i_1 \square g^j_2)(b\theta[1:m_1]) + \langle b\theta[m_1+1:m_1+m_2], a^A_i\rangle 
\end{align*}
\end{proof}

Then, by noting that an infimal convolution of two linear functions over bounded norm domain is convex, we show how to upper bound the infimal convolution with a linear term: 
\begin{restatable}{theorem}{thminfconvupperbound}
\label{thm:inf-conv-upper-bound}
Suppose the norm $||\cdot||$ is some $p$-norm where $p \neq 1$ and $p \neq \infty$. Fix $\theta$ where $||\theta||_{*} \le \alpha_A + \alpha_P$. Then, for any $b \in [0,1]$
    \begin{align*}
        (g^i_1 \square g^j_2)(b\theta) \le \left(\frac{b}{\alpha_A + \alpha_P} \right) (||\theta||_{*}\min(\alpha_A,\alpha_P)||x^A_i-x^P_j|| + \langle\theta, \alpha_A x^A_i + \alpha_P x^P_j\rangle) -\min(\alpha_A,\alpha_P)||x^A_i-x^P_j||
    \end{align*}
\end{restatable}
\begin{proof}
    Because Lemma \ref{lem:inf-conv-affine} tells us that the infimal convolution of $g^i_1$ and $g^j_2$ is convex, we know that $(g^i_1 \square g^j_2)(b\overline{\theta}_{*})$ must be convex in $b$. By convexity, we have that for any $b,b' \in[0,\alpha_A + \alpha_P]$ and $t \in [0,1]$
    \[
        (g^i_1 \square g^j_2)(((1-t) b + t b')\overline{\theta}_{*}) \le (1-t) (g^i_1 \square g^j_2)(b \overline{\theta}_{*}) + t (g^i_1 \square g^j_2)(b' \overline{\theta}_{*}).
    \]
    When we set $(b,b')=(0, \alpha_A + \alpha_P)$ and use the above upper bound, we get for any $t \in [0,1]$
    \begin{align*}
        (g^i_1 \square g^j_2)(t (\alpha_A + \alpha_P)\overline{\theta}_{*}) &\le -(1-t)\min(\alpha_A,\alpha_P)||x^A_i-x^P_j|| + t\langle\overline{\theta}_{*}, \alpha_A x^A_i + \alpha_P x^P_j\rangle\\
        &= t (\min(\alpha_A,\alpha_P)||x^A_i-x^P_j|| + \langle\overline{\theta}_{*}, \alpha_A x^A_i + \alpha_P x^P_j\rangle) -\min(\alpha_A,\alpha_P)||x^A_i-x^P_j|| 
    \end{align*}
    due to Lemma \ref{lem:inf-conv-at-zero} and \ref{lem:inf-conv-tight}.
    
    In other words, given any $\theta$ where $||\theta||_{*}\le \alpha_A + \alpha_P$, we can upper bound the infimal convolution as
    \begin{align*}
        (g^i_1 \square g^j_2)(b\theta) &= (g^i_1 \square g^j_2)(b||\theta||_{*}\overline{\theta}_{*})\\
        &=(g^i_1 \square g^j_2)\left(\frac{b ||\theta||_{*}}{\alpha_A + \alpha_P} (\alpha_A + \alpha_P) \overline{\theta}_{*}\right)\\
        &\le b \left(\frac{||\theta||_{*}}{\alpha_A + \alpha_P} \right) (\min(\alpha_A,\alpha_P)||x^A_i-x^P_j|| + \langle\overline{\theta}, \alpha_A x^A_i + \alpha_P x^P_j\rangle) -\min(\alpha_A,\alpha_P)||x^A_i-x^P_j|| \\
        &= \left(\frac{b}{\alpha_A + \alpha_P} \right) \left(\min(\alpha_A,\alpha_P)||\theta||_{*}||x^A_i-x^P_j|| + \langle\theta, \alpha_A x^A_i + \alpha_P x^P_j\rangle\right) -\min(\alpha_A,\alpha_P)||x^A_i-x^P_j||
    \end{align*}
\end{proof}

Now, using the fact that the infimal convolution of linear functions is convex which we prove in Appendix \ref{subsec:infimal-convolution}, we show how to upperbound the supremum term.
\thmsupupper*
\begin{proof}
For simplicity, we drop our dependence on $y$ and show an upper bound on 
\begin{align*}
    &R := \sup_{(x,a)} h(\theta, (x, a)) - \alpha_A ||x^A_i - x ||_{p} -  \alpha_P||x^P_j - x ||_{p} - \alpha_A \kappa_A ||a^A_i - a||_{p'}.
\end{align*}
Note that all our arguments here are based upon some fixed $\theta$, so if $y=+1$, proceed with the original $\theta$, and for $y=-1$, proceed with a new fixed $\theta' = -\theta$. 

Because $f$ is a convex function, its biconjugate is itself, so
    \[
    \sup_{b \in [0,1]} -f^*(b) + b \cdot X = f(X). 
    \]
Therefore, we have
\begin{align*}
    &\sup_{(x,a)} h(\theta, (x, a)) -\alpha_A||x -  x^A_i||_{p}  - \alpha_P||x-x^P_j||_{p} - \alpha_A \kappa_A ||a^A_i - a||_{p'} \\
    &=\sup_{b \in [0,1]} - f^*(b) + (g^i_1 \square g^j_2)(b\theta_1) + \langle b\theta_2, a^A_i\rangle\\
    &\le \sup_{b \in [0,1]} - f^*(b) + \left(\frac{b}{\alpha_A + \alpha_P} \right) \left(\min(\alpha_A,\alpha_P)||\theta_1||_{*}||x^A_i-x^P_j|| + \langle\theta_1, \alpha_A x^A_i + \alpha_P x^P_j\rangle\right) \\
    &- \min(\alpha_A, \alpha_P) || x^A_i - x^P_j||_{p} + b \langle \theta_2, a^A_i\rangle \\
    &= \sup_{b \in [0,1]} - f^*(b) \\
    &+ b\left(\left(\frac{1}{\alpha_A + \alpha_P} \right) \left(\min(\alpha_A,\alpha_P)||\theta_1||_{*}||x^A_i-x^P_j|| +  \langle\theta_1, \alpha_A x^A_i + \alpha_P x^P_j\rangle\right) + \langle \theta_2, a^A_i\rangle\right) - \min(\alpha_A, \alpha_P) || x^A_i - x^P_j||_{p}\\
    &= f\left(\left(\frac{\min(\alpha_A,\alpha_P)||\theta_1||_{*}||x^A_i-x^P_j|| +  \langle\theta, \alpha_A x^A_i + \alpha_P x^P_j\rangle}{\alpha_A + \alpha_P} \right) + \langle \theta_2, a^A_i\rangle\right) - \min(\alpha_A, \alpha_P) || x^A_i - x^P_j||_{p}.
\end{align*}
The first inequality follows from Theorem \ref{thm:inf-conv-upper-bound}.
\end{proof}

\begin{lemma}
\label{lem:arg_min_second_term}
\[
    \inf_x (\alpha_A ||x-x^A_i|| + \alpha_P ||x-x^P_j||)
    = \min(\alpha_A, \alpha_P)||x^A_i - x^P_j||.
\]
and when $\alpha_A < \alpha_P$, the infimum is achieved at $x=x^P_j$ and otherwise at $x^A_i$.
\end{lemma}
\begin{proof}
    \begin{align*}
        &\inf_x \sup_{||q_1||_{*} \le \alpha_A} \langle q_1, x-x^A_i \rangle + \sup_{||q_2||_{*}\le \alpha_P}\langle q_2, x-x^P_j\rangle \\
        &=\inf_x \sup_{\substack{||q_1||_{*} \le \alpha_A,\\ ||q_2||_{*}\le \alpha_P} }   \langle q_1+q_2, x \rangle + \langle q_1, -x^A_i\rangle + \langle q_2, -x^P_j\rangle 
    \end{align*}
    
    We are able to swap $\inf$ and $\sup$ due to proposition 5.5.4 of \cite{bertsekas2009convex}.
    \begin{align*}
        &=\sup_{\substack{||q_1||_{*} \le \alpha_A,\\ ||q_2||_{*}\le \alpha_P} } \inf_x \langle q_1+q_2, x \rangle + \langle q_1, -x^A_i\rangle + \langle q_2, -x^P_j\rangle \\
        &=\sup_{||q||_{*} \le \min(\alpha_A,\alpha_P) } \langle q, -x^A_i + x^P_j\rangle \\
        &=\min(\alpha_A, \alpha_P) ||x^A_i - x^P_j||.
    \end{align*}
    The second inequality holds true because The sum of two norms has to be non-negative, and unless $q_1 = q_2$, the $\inf$ term can be made arbitrarily small, meaning we need to set $q_1 = -q_2$. 
\end{proof}

%\thmgapsup*
\lemholderlem*
\begin{proof}
    Fix $i,j, \theta, \alpha_A, \alpha_P$. 
% We write $U(x,a) = - \alpha_A||x - x^A_i|| - \alpha_P||x - x^P_j|| - \kappa_A \alpha_A|a^A_i - a|$, meaning $(x^*, a^*) = \arg\max_{(x,a)} f(\langle \theta, (x,a) \rangle) +U(x, a)$. 

% From Theorem \ref{thm:sup_upper}, we have
%     \begin{align*}
%          f(\langle \theta, (x^*,a^*)\rangle) + U(x^*, a^*) \le f(\langle \theta, (\bar{x},a^A_i)\rangle) + U(\hat{x}, a^A_i)
%     \end{align*}
% where  
% \begin{align*}
%     \bar{x} =  \frac{\alpha_A x^A_i + \alpha_P x^P_j + \hat{\alpha}v(\theta)||x^A_i - x^P_j||}{\alpha_A + \alpha_P}.
% \end{align*} 
% Therefore, we have
% \begin{align*}
%     f(\langle \theta, (\hat{x},a^A_i)\rangle) + U(\hat{x},a^A_i) \le f(\langle \theta, (x^*,a^*)\rangle) + U(x^*, a^*) \le f(\langle \theta, (\bar{x},a^A_i)\rangle) + U(\hat{x}, a^A_i).
% \end{align*}
% In other words, 
% \[
%     (f(\langle \theta, (x^*,a^*)\rangle) + U(x^*, a^*)) - (f(\langle \theta, (\hat{x}, a^A_i)\rangle) + U(\hat{x}, a^A_i)) \le f(\langle \theta, (\bar{x},a^A_i)\rangle) - f(\langle \theta, (\hat{x},a^A_i)\rangle),
% \]
% meaning it is sufficient for us to bound $f(\langle \theta, (\bar{x},a^A_i)\rangle) - f(\langle \theta, (\hat{x},a^A_i)\rangle)$

H\"{o}lder's inequality gives us
\begin{align*}
    \left| \sum_{c \in [m_1]} \theta[c] (x^A_i-x^P_j)[c]\right| 
    &\le \sum_{c \in [m_1]} \left|\theta[c] (x^A_i-x^P_j)[c]\right|
    \le ||\theta||_{*}||x^A_i-x^P_j||
\end{align*}

Suppose $\alpha_A < \alpha_P$, meaning $\hat{x} = x^P_j$
\begin{align*}
    &f\Bigg(\Bigg(\frac{\min(\alpha_A,\alpha_P)||\theta_1||_{*}||x^A_i-x^P_j||}{\alpha_A + \alpha_P} + \frac{\langle\theta_1, \alpha_A x^A_i + \alpha_P x^P_j\rangle}{\alpha_A + \alpha_P} + \langle \theta_2, a^A_i \rangle \Bigg) - f(\langle \theta, (\hat{x},a^A_i)\rangle)\\
    &\le \left|\frac{\alpha_A||\theta||_{*}||x^A_i-x^P_j|| +  \langle\theta, \alpha_A x^A_i + \alpha_P x^P_j\rangle}{\alpha_A + \alpha_P} + \langle \theta_2, a^A_i \rangle  - \langle \theta, (\hat{x},a^A_i)\rangle \right|\\
    &= \left|\frac{\alpha_A||\theta||_{*}||x^A_i-x^P_j|| +  \langle\theta, \alpha_A x^A_i + \alpha_P x^P_j\rangle}{\alpha_A + \alpha_P} - \frac{\alpha_A \sum_{c \in [m_1]} \theta[c] (x^P_j - x^A_i)[c] +  \langle\theta, \alpha_A x^A_i + \alpha_P x^P_j\rangle}{\alpha_A + \alpha_P}\right|\\
    &= \left|\frac{\alpha_A||\theta||_{*}||x^A_i-x^P_j|| - \alpha_A(\sum_{c \in [m_1]} \theta[c] (x^P_j - x^A_i)[c])}{\alpha_A + \alpha_P}\right|\\
    &\le \frac{2\alpha_A||\theta||_{*}||x^A_i-x^P_j||}{\alpha_A + \alpha_P}\\
    &\le 2\alpha_A||x^A_i-x^P_j||
\end{align*}
where the first inequality follows from $f$'s 1-Lipschitzness --- i.e. $|f(x) - f(x') \le |x-x'|$. The same argument works when $\alpha_A \ge \alpha_P$.
\end{proof}

\corsupapproxeverything*
\begin{proof}
It suffices to show that
    \begin{align*}
    &\max_{y \in \{\pm 1\}} \kappa_P \alpha_P |y^P_j-y| + f(\langle y\theta, (\hat{x}_{i,j}, a^A_i) \rangle) - \alpha_A d^i_A(\hat{x}_{i,j},a^A_i)  -  \alpha_P d^j_P(\hat{x}_{i,j},y)\\
    &= \max_{y \in \{y^P_j, -y^P_j\}}\kappa_P \alpha_P |y^P_j-y| + f(\langle y\theta, (\hat{x}_{i,j}, a^A_i) \rangle) - \alpha_A d^i_A(\hat{x}_{i,j},a^A_i)  -  \alpha_P d^j_P(\hat{x}_{i,j},y)\\
    &=f(\langle y^P_j \theta, (\hat{x}_{i,j}, a^A_i) \rangle) + \max(y^P_j\langle\theta, (\hat{x}_{i,j}, a^A_i) \rangle-\alpha_P\kappa_P,0) -\hat{\alpha}||x^A_i -x^P_j||.
\end{align*}
\end{proof}

\thmmainthm*
\begin{proof}
First, we show that we can solve problem \eqref{eqn:convex-approx-prob} by solving two convex optimization problems. The only part in the problem that is not convex is the use of $\hat{x}_{i,j}$ and $\hat{\alpha}$. However, solving \eqref{eqn:convex-approx-prob} is equivalent to solving the following two problems and taking the one with a smaller objective value: \begin{enumerate}
    \item Set $\hat{x}_{i,j} = x^P_j$ for all $i,j$ and $\hat{\alpha} = \alpha_A$ along with a new constraint $\alpha_A < \alpha_P$.
    \item Set $\hat{x}_{i,j} = x^A_i$ for all $i,j$ and $\hat{\alpha} = \alpha_P$ along with a new constraint $\alpha_A \ge \alpha_P$.
\end{enumerate}

By solving two convex optimizations above and choosing the solution out of those two that achieves the smaller objective value, suppose we have obtained the optimal solution to problem \eqref{eqn:convex-approx-prob}, which we call $(\theta, \alpha_A, \alpha_P, \{\beta_i\}, \{\beta'\})$. Now we claim that the following solution 
\[
    (\theta^*, \alpha^*_A, \alpha^*_P, \{\beta^*_i\}, \{{\beta'}^*\}) = (\theta, \alpha_A, \alpha_P, \{\beta_i + \hat{\alpha} b\}, \{\beta' + \hat{\alpha} b\})
\]
where $b= \max_{i \in [n_A], j \in n_P]} ||x^A_i -x^P_j||$ is a dual feasible solution:
\begin{align*}
    &\left(f(y^P_j\langle\theta, (\hat{x}_{i,j}, a^A_i) \rangle) + \max(y^P_j\langle\theta, (\hat{x}_{i,j}, a^A_i) \rangle-\alpha_P\kappa_P,0) -\hat{\alpha}||x^A_i -x^P_j||\right) \le \beta_i + \beta'_j\\
    \implies &\left(f(y^P_j\langle\theta, (\hat{x}_{i,j}, a^A_i) \rangle) + \max(y^P_j\langle\theta, (\hat{x}_{i,j}, a^A_i) \rangle-\alpha_P\kappa_P,0) -\hat{\alpha}||x^A_i -x^P_j||\right) + 2\hat{\alpha} ||x^A_i - x^P_j|| \le \beta_i + \beta'_j + 2 \hat{\alpha}b\\
    \implies &\max_{y \in \{\pm 1\}}\sup_{(x,a)} \left(\ell(\theta, (x, a, y)) - \alpha_A d^i_A(x,a)  -  \alpha_P d^j_P(x,y)\right) \le \beta^*_{i} + {\beta'}^*_j
\end{align*}
where the last implication follows from Corollary~\ref{cor:sup-approx-everything}.

Finally, we argue that the objective value of $(\theta^*, \alpha^*_A, \alpha^*_P, \{\beta^*_i\}, \{{\beta'}^*\})$ is at most $2\hat{\alpha}b$ more than that of the optimal dual solution. 

\begin{lemma}
Any feasible dual solution is a feasible solution to problem \eqref{eqn:convex-approx-prob}. In other words, the objective value of the optimal solution to problem \eqref{eqn:convex-approx-prob} must be less than or equal to the objective value of the dual solution.
\end{lemma}
\begin{proof}
Consider any feasible dual solution $(\theta, \alpha_A, \alpha_P, \{\beta_i\}, \{\beta'\})$. Then, we must have
    \begin{align*}
        &f(y^P_j\langle\theta, (\hat{x}_{i,j}, a^A_i) \rangle) + \max(y^P_j\langle\theta, (\hat{x}_{i,j}, a^A_i) \rangle-\alpha_P\kappa_P,0) -\hat{\alpha}||x^A_i -x^P_j|| \\
        &\le \max_{y \in \{\pm 1\}} \sup_{(x,a)} \left(\ell(\theta, (x, a, y)) - \alpha_A d^i_A(x,a)  -  \alpha_P d^j_P(x,y)\right) \le \beta_i + \beta'_j,
    \end{align*}
     meaning the feasible space of problem \eqref{eqn:convex-approx-prob} is a superset of the feasible space of the dual problem.
\end{proof}

\newcommand{\dual}{\text{dual}}
Write the optimal dual solution as $(\theta^{\dual}, \alpha^\dual_A, \alpha^\dual_P, \{\beta_i^\dual\}_i, \{{\beta'}^\dual_j\}_j)$ and recall the optimal solution to problem \eqref{eqn:convex-approx-prob} is $(\theta, \alpha_A, \alpha_P, \{\beta_i\}, \{\beta'\})$. Then, we have
\begin{align*}
    &\alpha_A r_A + \alpha_P r_P + \frac{1}{n_A}\sum_{i\in [n_a]} \beta_{i} + \frac{1}{n_P} \sum_{j \in [n_P]} \beta'_j \le \alpha^\dual_A r_A + \alpha^\dual_P r_P + \frac{1}{n_A}\sum_{i\in [n_a]} \beta^\dual_{i} + \frac{1}{n_P} \sum_{j \in [n_P]} {\beta'}^\dual_j \\
    \implies &\alpha^*_A r_A + \alpha^*_P r_P + \frac{1}{n_A}\sum_{i\in [n_a]} \beta^*_{i} + \frac{1}{n_P} \sum_{j \in [n_P]} {\beta'}^*_j -2\hat{\alpha}b \le \alpha^\dual_A r_A + \alpha^\dual_P r_P + \frac{1}{n_A}\sum_{i\in [n_a]} \beta^\dual_{i} + \frac{1}{n_P} \sum_{j \in [n_P]} {\beta'}^\dual_j.
\end{align*}

Because of the strong duality due to Theorem \ref{thm:our-duality}, the objective of the optimal dual solution is exactly the objective of the optimal solution to the original problem \eqref{eqn:disjoint-learning-orig}. This concludes the proof of Theorem \ref{thm:main-thm}.
\end{proof}

\subsection{Infimal Convolution}
\label{subsec:infimal-convolution}
Now, we prove a few lemmas regarding the infimal convolution of two linear functions\footnote{Readers more interested in the properties of infimal convolution may refer to \cite{stromberg1994study}.}. Since the domain of $g^i_1$ and $g^j_2$ invovles the same $p$-norm, we elide $p$ in the following lemmas.

\begin{lemma}
\label{lem:inf-conv-affine}
$(g^i_1 \square g^j_2)(\theta)$ is convex in $\theta$.
\end{lemma}
\begin{proof}
In order to show a function is convex, it suffices to show that its epigraph is convex. 
Note that epigraphs of $g^i_1$ and $g^j_2$ are both convex:
\[
    S_1 = \epi g^i_1 = \{(q,r): ||q||_{*} \le \alpha_A, r \ge g^i_1(q) \}
\]
\[
    S_2 = \epi g^j_2 = \{(q,r): ||q||_{*} \le \alpha_P, r \ge g^j_2(q) \}.
\]
Note that the epigraph of $(g^i_1 \square g^j_2)(\theta)$ is the Minkowski sum of $S_1$ and $S_2$ \cite{stromberg1994study}:
\[
    S_3 = \epi (g^i_1 \square g^j_2) = \{(x_1 + x_2, r_1 + r_2): (x_1, r_1) \in S_1, (x_2, r_2) \in S_2\}.
\]
For any $(x_1 + x_2, r_1 + r_2) \in S_3$ and $(x'_1 + x'_2, r'_1 + r'_2) \in S_3$ where $(x_1, r_1), (x'_1, r'_1) \in S_1$ and $(x_2, r_2), (x'_2, r'_2) \in S_2$, the convex combination with $t \in [0,1]$
\[(t(x_1 + x_2) + (1-t)(x'_1 + x'_2), t(r_1 + r_2) + (1-t) (r'_1 + r'_2)\] must belong in $S_3$ because
$(tx_1 + (1-t)x'_1, tr_1 + (1-t)r'_1) \in S_1$ and $(tx_2 + (1-t)x'_2, tr_2 + (1-t)r'_2) \in S_2$ due to the convexity of $S_1$ and $S_2$.

\end{proof}

\begin{lemma}
\label{lem:inf-conv-at-zero}
\[
        (g^i_1 \square g^j_2)(0) = -\min(\alpha_A, \alpha_P) || x^P_j-x^A_i||.
\]
\end{lemma}
\begin{proof}
    \begin{align*}
        (g^i_1 \square g^j_2)(0) &= \inf_{q: ||q||_{*} \le \min(\alpha_A, \alpha_P)} g^i_1(q) + g^j_2(-q) \\
        &= \inf_{q: ||q||_{*} \le \min(\alpha_A, \alpha_P)} \langle q, x^A_i\rangle - \langle q, x^P_j \rangle \\
        &= -\sup_{q: ||q||_{*} \le \min(\alpha_A, \alpha_P)} \langle q, -x^A_i + x^P_j \rangle \\
        &= -\min(\alpha_A, \alpha_P) || x^A_i- x^P_j||.
    \end{align*}
\end{proof}

Now, for any $q$, we write \[
        v(q) = \arg\max_{v: ||v|| \le 1} \langle v, q \rangle.
    \]
Note that $\langle v(q), q \rangle = || q||_{*}$. In words, $v(q)$ is the vector whose inner product with $q$ evaluates to the dual norm of $q$. Note that for any scalar $c > 0$, $v(q) = v(c q)$, meaning only the direction matters.

In the lemma below, we show that given two different directions $(q,q')$, we must have $v(q) \neq v(q')$.

\begin{lemma}
\label{lem:dual-norm-normal-vector}
    Suppose the norm $||\cdot||$ is some $p$-norm where $p \neq 1$ and $p \neq \infty$, meaning corresponding dual norm $||\cdot||_{p,*}$ is $r$-norm where $r\neq 1$ and $r \neq \infty$. Given $q$ and $q'$ where $||q||_{*}=||q'||_{*} = 1$ and $q \neq q'$, we must have $v(q) \neq v(q')$.
\end{lemma}
\begin{proof}
     For any $q$ where $||q||_{*}=1$, let's consider $v(q) = \max_{v: ||v|| \le 1} \langle v, q \rangle$. Because the linear objective forces the optimal solution be at the boundary of the feasible convex set, it is equivalent to solving $\max_{v: ||v|| = 1} \langle v, q \rangle$. Lagrange multiplier approach yields the following conditions for the optimal solution:
    \begin{align}
    \begin{split}
        q[i] + \lambda \cdot \text{sign}(v(q)[i]) \cdot \left(\frac{|v(q)[i]|}{||v(q)||}\right)^{p-1} &= 0 \quad \forall i \in [n]\\
        ||v(q)|| &= 1
    \end{split}
    \label{eqn:lagr-station}
    \end{align}
   where $\lambda$ corresponds to the Lagrange multiplier. 
   
   Consider the following two unnormalized vectors $v^{+1}$ and $v^{-1}$:
   \begin{align*}
       v^{+1}(q) &= \left(\text{sign}(q[i]) \cdot \left| |q[1]|^{\frac{1}{p-1}}\right|, \dots, \text{sign}(q[n]) \cdot \left| |q[n]|^{\frac{1}{p-1}}\right|\right) \\
       v^{-1}(q) &= \left(\text{sign}(-q[i]) \cdot \left| |q[1]|^{\frac{1}{p-1}}\right|, \dots, \text{sign}(-q[n]) \cdot \left| |q[n]|^{\frac{1}{p-1}}\right|\right).
   \end{align*} 
   The solutions to the equations in \eqref{eqn:lagr-station} are the normalized $\frac{v^{+1}(q)}{||v^{+1}(q)||}$ and $\frac{v^{-1}(q)}{||v^{-1}(q)||}$, meaning they are the local optima. 
   
   Because $\langle \frac{v^{+1}(q)}{||v^{+1}(q)||}, q \rangle = - \langle \frac{v^{-1}(q)}{||v^{-1}(q)||}, q \rangle$ and $\langle \frac{v^{+1}(q)}{||v^{+1}(q)||}, q \rangle > 0$, we must have that 
    \[
     v(q) = \arg\max_{v: ||v|| \le 1} \langle v, q\rangle = \frac{v^{+1}(q)}{||v^{+1}(q)||}.
    \]
   
   Hence, for any two different directions $q$ and $q'$, we must have that $\frac{v^{+1}(q)}{||v^{+1}(q)||}$ will be different by construction, as long as $p \neq 1$ or $p \neq \infty$. Hence, $v(q) \neq v(q')$.
\end{proof}

\begin{corollary}
\label{cor:dual-norm-normal-vector}
Suppose the norm $||\cdot||$ is some $p$-norm where $p \neq 1$ and $p \neq \infty$.
 For any $q$ where $||q||_{*} = \alpha$, we have that for any other $q'$ where $q' \neq q$ and $||q'||_{*} \le \alpha$,
    \begin{align*}
        \langle v(q), q\rangle > \langle v(q), q'\rangle.
    \end{align*}
\end{corollary}
\begin{proof}
    As said in Section \ref{sec:prelim}, given any vector $q$, we'll write $\overline{q}_{*} = \frac{q}{||q||_{*}}$. If $\overline{q}_{*} = \overline{q'}_{*}$, then there exists some scalar $c > 1$ such that $q = c q'$ since $||q||_{*} > ||q'||_{*}$. Then, we must have
    \begin{align*}
        \langle v(q), q \rangle = c \langle v(q), q' \rangle > \langle v(q), q' \rangle
    \end{align*}
    as $v(q) = v(q')$ in this case.
    
    Now, in the case where $\overline{q}_{*} \neq \overline{q'}_{*}$, we see that
    \begin{align*}
        \langle v(q), q \rangle = \alpha  \ge ||q'||_{*}  = \langle v(q'), q' \rangle > \langle v(q), q' \rangle.
    \end{align*}
\end{proof}

\begin{lemma}
\label{lem:inf-conv-tight}
Suppose the norm $||\cdot||$ is some $p$-norm where $p \neq 1$ and $p \neq \infty$.
Fix some direction $\overline{\theta}_{*}$ where $||\overline{\theta}_{*}||_{*} = 1$. Then,
\begin{align*}
    (g^i_1 \square g^j_2)((\alpha_A + \alpha_P)\overline{\theta}_{*}) = \langle \overline{\theta}_{*}, \alpha_A x^A_i + \alpha_P x^P_j \rangle.
\end{align*}
\end{lemma}
\begin{proof}
    We first claim that when given $(\alpha_A + \alpha_P)\overline{\theta}_{*}$, there exists only one pair $(q_1,q_2)$ such that $||q_1||_{*} \le \alpha_A$, $||q_2||_{*} \le \alpha_P$, and $q_1 + q_2 = (\alpha_{a} + \alpha_P)\overline{\theta}_{*}$: namely, 
    \[
        q_1 = \alpha_A \overline{\theta}_{*} \quad\text{and}\quad q_2 = \alpha_P \overline{\theta}_{*}.
    \]
    By construction, $||q_1||_{*} = \alpha_A$, $||q_2||_{*} = \alpha_P$, and $q_1 + q_2 = (\alpha_A + \alpha_P)\overline{\theta}_{*}$. 
    
    Now, for the sake of contradiction, suppose there exists another $(q'_1, q'_2)$ such that the above condition holds true. Because $q'_1 + q'_2 = (\alpha_A + \alpha_P)\overline{\theta}$, let's say that
    $q'_1 = q_1 + u$ and $q'_2 = q_2 - u$ for some $u \neq 0$. However, we argue that it must be the case that either $||q'_1||_{*} > \alpha_A$ or $||q'_2||_{*} > \alpha_P$. Without loss of generality, suppose $||q_1 + u||_{*} = \alpha_A -\epsilon$ for some $\epsilon \ge 0$.
    
    Now, consider $v(\overline{\theta}) = v(q_1) = v(q_2)$. Corollary \ref{cor:dual-norm-normal-vector} tells us that for any other $q'$ where $||q'||_{*} \le \alpha_A$, 
    \[
        \langle v(\overline{\theta}_{*}), q'\rangle < \langle v(\overline{\theta}_{*}), q_1 \rangle = \alpha_A.
    \]
    Because the dual norm of $q_1+u$ is still bounded by $\alpha_A$, we have
    \begin{align*}
        \langle v(\overline{\theta}_{*}), q_1 + u \rangle &< \langle v(\overline{\theta}_{*}), q_1 \rangle\\
        \langle v(\overline{\theta}_{*}), u \rangle &< 0.
    \end{align*}
    
   Then, we must have
   \begin{align*}
       ||q'_2||_{*} = \langle v(q'_2), q_2 - u\rangle > \langle v(\overline{\theta}_{*}), q_2 - u\rangle = \alpha_P - \langle v(\overline{\theta}_{*}), u \rangle > \alpha_P,
   \end{align*}
   giving us the needed contradiction.
   
   Therefore, because there's only pair $(q_1,q_2) = (\alpha_A \overline{\theta}_{*}, \alpha_P \overline{\theta}_{*})$ where $||q_1||_{*} \le \alpha_A$, $||q_2||_{*} \le \alpha_P$, and $q_1 + q_2 = (\alpha_A + \alpha_P)\overline{\theta}_{*}$, we must have
   \begin{align*}
       (g^i_1 \square g^j_2)((\alpha_A + \alpha_P)\overline{\theta}_{*}) &= g^i_1(\alpha_A \overline{\theta}_{*}) + g^j_2(\alpha_{P}\overline{\theta}_{*})\\
       &= \langle \overline{\theta}_{*}, \alpha_A x^A_i + \alpha_P x^P_j \rangle
   \end{align*}  
\end{proof}

\section{Simplifications via Additional Assumptions}
\label{sec:opt}
\subsection{Approximation}
\label{subsec:approx}
We will first try to reformulate the original problem by making some structural assumption about the optimal transport $\pi^y_{i,j}(x,a)$. Because it is an optimal transport, we most likely have that for every $(x,a,y)$ whose measure is non-zero (i.e. $\pi^y_{i,j}(x,a)> 0$), its distance to $(x^A_i, a^A_i)$ and $(x^P_j, y^P_j)$ should be small. In other words, we most likely have that for any $(i,j)$ where $||x^A_i - x^P_j||$ is big, $\pi^y_{i,j}(x,a)$ will be zero. Therefore, we assume that for every $i \in [n_A]$, we only consider its $k$-closest neighbors out of $\{x^P_j\}_{j \in [n_P]}$ and do the same for $j \in [n_P]$. We will denote this set of pairs by $M$
\begin{align*}
    M = \left\{(i,j): \parbox{.75\textwidth}{$x^A_i$ is one of $x^P_j$'s $k$-nearest neighbors among $\{x^A_{i'}\}_{i'}$ or $x^P_j$ is one of $x^A_i$'s $k$-nearest neighbors among $\{x^P_{j'}\}_{j'}$.}\right\}
\end{align*}

Noting that the dual constraint for each $i \in [n_A], j\in[n_P],$ and $y \in \cY$ corresponds to the primal variable $\pi^y_{i,j}$, we may have the primal variables be defined only over pairs $(i,j) \in M$ and have the dual constraints be defined only over $M$ consequently. Then, after multiplying the objective by $n_A n_P$ with some rearranging, the dual problem becomes 
\begin{align*}
    &\min_{\substack{\theta, \alpha_A, \alpha_P, \\ \{\beta_{i}, \beta'_j\}_{(i,j) \in M}}} n_A n_P(\alpha_A r_A + \alpha_P r_P) + \sum_{(i,j) \in M} (\beta_{i} + \beta'_j)\\
    &\text{s.t.} \max_{y \in \cY}\sup_{(x,a)} \left(\ell(\theta, (x, a, y)) - \alpha_A d^i_A(x,a)  -  \alpha_P d^j_P(x,y)\right) \le \beta_{i} + \beta'_j \quad \forall (i,j) \in M.
\end{align*}
In the case where the k-nearest-neighbor graph $M$ between $S_A$ and $S_P$ is nicely structured\footnote{More formally, this is equivalent to assuming that there exists a feasible solution to the following system following linear equations. Suppose $A$ is a $|M| \times (n_A + n_P)$ matrix where for every $l$th pair $(i,j)$ in $M$, $M[l, i]=1$ and $M[l, n_A +j]=1$. And $b$ is a vector of length $|M|$ where for every $l$th pair $(i,j) \in M$, $b[l] = \max_{y \in \cY} \sup_{(x,a)} \left(\ell(\theta, (x, a, y)) - \alpha_A d^i_A(x,a)  -  \alpha_P d^j_P(x,y)\right).$ Our assumption is equivalent to assuming that there exists a vector $x$ of length $n_A + n_P$ such that $Ax=b$ or equivalently, $A$ is left-invertible. Note that the very first $n_A$ coordinates correspond to $\{\beta_i\}$ and the last $n_P$ coordinates correspond to $\{\beta'_j\}$.}, we should be able to find $\beta_i$ and $\beta_j$ such that for each $(i,j) \in M$, $\max_{y \in \cY} \sup_{(x,a)} \left(\ell(\theta, (x, a, y)) - \alpha_A d^i_A(x,a)  -  \alpha_P d^j_P(x,y)\right) = \beta_{i} + \beta'_j$. This assumption allows us to re-write the problem as 
\begin{align*}
    &\min_{\alpha_A, \alpha_P, \theta} n_An_P(\alpha_A r_A + \alpha_P r_P) + \sum_{(i,j) \in M} \max_{y \in \cY}\sup_{(x,a)} \big(\ell(\theta, (x, a, y)) - \alpha_A d^i_A(x,a)  -  \alpha_P d^j_P(x,y)\big).
\end{align*}

Using our approximation of the supremum term as in Theorem \ref{thm:main-thm} and the above fact, the problem then becomes
\begin{align}
    &\min_{\alpha_A, \alpha_P, \theta_1, \theta_2} (\alpha_A r_A + \alpha_P r_P) + \frac{1}{n_A n_P} \sum_{(i,j) \in M} (f(y^P_j\langle\theta, (\hat{x}_{i,j}, a^A_i) \rangle) + \max(y^P_j\langle\theta, (\hat{x}_{i,j}, a^A_i) \rangle-\alpha_P\kappa_P,0) -\hat{\alpha}||x^A_i -x^P_j||) \label{eqn:final-opt-prob}\\
    &\text{s.t.} ||\theta_1||_{*} \le \alpha_A + \alpha_P, ||\theta_2||_{*} \le \kappa_A \alpha_A. \nonumber
\end{align}
Note that because we have restricted our attention only to pairs who are close to one another, the additive approximation error which amounts to $2\hat{\alpha}\sum_{(i,j) \in M} ||x^A_i - x^P_j||$ in the objective must be small. It can be further shown that the above problem can be solved via solving two instances of convex optimization, which we solve simply via projected gradient descent.

\subsection{Projected Gradient Descent}
\label{app:projected-gradient-descent}
To solve the optimization problem \eqref{eqn:final-opt-prob}, we employ first-order projected gradient descent. In order to handle $\hat{\alpha}=\min(\alpha_A, \alpha_P)$. We can just solve the optimization twice: once with $\alpha_A < \alpha_P$ as one of the constraints and the other time with $\alpha_A \ge \alpha_P$. Suppose $\hat{\alpha} = \alpha_A$, meaning $\hat{x}_{i,j} = x^P_j$. Then we write the objective function as
\begin{align*}
    &\Omega^A(\alpha_A, \alpha_P, \theta) \\
    &= (\alpha_A r_A + \alpha_P r_P) +\frac{1}{n_A n_P} \sum_{(i,j) \in M}(f(y^P_j\langle\theta, (x^P_j, a^A_i) \rangle) \\
    &+ \max(y^P_j\langle\theta, (x^P_j, a^A_i) \rangle-\alpha_P\kappa_P,0) -\alpha_A||x^A_i -x^P_j|| )
\end{align*}
and the constraint set is
\begin{align*}
    C^A &= \{(\alpha_A, \alpha_P, \theta): ||\theta_1||_{*} \le \alpha_A + \alpha_P, \\
    &||\theta_2|| \le \kappa_A \alpha_A, \alpha_A < \alpha_P\}.
\end{align*}
Similarly, when $\hat{\alpha} = \alpha_P$, we write $\Omega^P(\alpha_A, \alpha_P, \theta)$ and $C^P$ where the $\alpha$ constraint is replaced by $\alpha_A \ge \alpha_P$. Note that in both cases, we have a convex optimization problem.
\begin{claim}
The objective functions $\Omega^A(\alpha_A, \alpha_P, \theta)$ and $\Omega^P(\alpha_A, \alpha_P, \theta)$ are convex in $(\alpha_A, \alpha_P, \theta)$. The constraint sets $C^A$ and $C^P$ are also convex in $(\alpha_A, \alpha_P, \theta)$. 
\end{claim}
Suppose we write $(\alpha'_A, \alpha'_P, \theta') = \arg\min_{(\alpha_A,\alpha_P, \theta) \in C^A} \Omega^A(\alpha_A, \alpha_P, \theta)$ and $(\alpha{''}_A, \alpha{''}_P, \theta{''}) = \arg\min_{(\alpha_A,\alpha_P, \theta) \in C^P} \Omega^P(\alpha_A, \alpha_P, \theta)$.

\begin{claim}
The optimal solution to problem \eqref{eqn:final-opt-prob} is $(\alpha'_A, \alpha'_P, \theta')$ if $\Omega^A(\alpha'_A, \alpha'_P, \theta') \le \Omega^P(\alpha{''}_A, \alpha{''}_P, \theta{''})$ and $(\alpha{''}_A, \alpha{''}_P, \theta{''})$ otherwise.
\end{claim}

Typical regularized models either constrain the norm of the parameter $\theta$ to be directly bounded by some constants specified initially or include the norm as part of the objective multiplied by some multiplicative penalty constant. However, our optimization problem is a hybrid of both as (1) the norms of the parameter $\theta$ are to be bounded by $\alpha_A$ and $\alpha_P$ but (2) $(\alpha_A,\alpha_P)$ are part of the optimization variables that are multiplied by some penalty constants $r_A$ and $r_P$ in the objective function. 

Nevertheless, the constraint set is convex so Euclidean projection can be solved via any convex solver, and in the case of $p=2$, we have exactly characterized a closed form solution of the output of the projection in Appendix \ref{app:projected-gradient-descent}. Therefore, in order to solve \eqref{eqn:final-opt-prob}, we can employ projected gradient descent (PGD): $(\alpha_A^{t+1}, \alpha_P^{t+1}, \theta^{t+1}) = \text{Project}_{C}\left((\alpha_A^{t}, \alpha_P^{t}, \theta^{t}) - \eta \nabla \Omega(\alpha_A^{t}, \alpha_P^{t}, \theta^{t})\right)$.
% \begin{align*}
%     &(\alpha_A^{t+1}, \alpha_P^{t+1}, \theta^{t+1}) \\
%     &= \text{Project}_{C}\left((\alpha_A^{t}, \alpha_P^{t}, \theta^{t}) - \eta \nabla \Omega(\alpha_A^{t}, \alpha_P^{t}, \theta^{t})\right).
% \end{align*}
It is well known that the rate of convergence for PGD is $O(\frac{1}{\sqrt{T}})$ with appropriately chosen step size $\eta$. We present the overall algorithm to solve problem \eqref{eqn:final-opt-prob} in Algorithm \ref{alg:drdj-alg} which can be found in Appendix \ref{app:projected-gradient-descent}. Write $\Omega(\alpha_A, \alpha_P, \theta) = \Omega^A(\alpha_A, \alpha_P, \theta)$ if $\alpha_A < \alpha_P$ and $\Omega^P(\alpha_A, \alpha_P, \theta)$ otherwise to denote the objective solution to problem \eqref{eqn:final-opt-prob}. Then, the optimal value $(\alpha^*_A, \alpha^*_P, \theta^*)$ of problem \eqref{eqn:final-opt-prob} is $(\alpha^*_A, \alpha^*_P, \theta^*) = \arg\min_{(\alpha_A, \alpha_P, \theta) \in C^A \cup C^P} \Omega(\alpha_A, \alpha_P, \theta)$.
\begin{restatable}{theorem}{drjoinalgconvergence}
 With appropriately chosen step size $\eta$, Algorithm \ref{alg:drdj-alg} returns $(\alpha_A, \alpha_P, \theta)$ such that $\Omega(\alpha_A, \alpha_P, \theta) \le \Omega(\alpha^*_A, \alpha^*_P, \theta^*) + O\left(\frac{1}{\sqrt{T}}\right)$.
\end{restatable}

\begin{algorithm}[H]
   \caption{Distributionally Robust Data Join}
   \label{alg:drdj-alg}
\begin{algorithmic}[1]
   \STATE {\bfseries Input:} $S_A$, $S_P$, $r_A$, $r_P$, $\kappa_A$, $\kappa_P$, $k$, $T$
   \STATE Run $k$-nearest neighbors on $S_A$ and $S_P$ to calculate the matching pairs $M$
   \STATE choose arbitrary $\theta, \alpha_A, \alpha_P$
   \STATE Set $\theta^1_A = \theta, \alpha^1_A = \alpha_A, \alpha^1_P = \alpha_P$
   \STATE Set ${\theta'}^1_A = \theta, {\alpha'}^1_A = \alpha_A, \alpha^1_P$
   \FOR{$i=1$ {\bfseries to} $T$}
    \STATE $( \alpha_A^{t+1}, \alpha_P^{t+1}, \theta^{t+1}) = \text{Project}_{C^A}\left((\alpha_A^{t}, \alpha_P^{t}, \theta^{t}) - \eta \nabla \Omega^A(\alpha_A^{t}, \alpha_P^{t}, \theta^{t})\right)$
    \STATE $({\alpha'}_A^{t+1}, {\alpha'}_P^{t+1}, {\theta'}^{t+1}) = \text{Project}_{C^P}\left(({\alpha'}_A^{t}, {\alpha'}_P^{t}, {\theta'}^{t}) - \eta \nabla \Omega^P({\alpha'}_A^{t}, {\alpha'}_P^{t}, {\theta'}^{t})\right)$
   \ENDFOR
   \STATE $\overline{\alpha_A} = \frac{1}{T} \sum_{t=1}^T \alpha_A^t, \overline{\alpha_P} = \frac{1}{T} \sum_{t=1}^T \alpha_P^t, \overline{\theta} = \frac{1}{T} \sum_{t=1}^T \theta^t$
   \STATE $\overline{\alpha_A}' = \frac{1}{T} \sum_{t=1}^T {\alpha'}_A^t, \overline{\alpha_P}' = \frac{1}{T} \sum_{t=1}^T {\alpha'}_P^t, \overline{\theta}' = \frac{1}{T} \sum_{t=1}^T {\theta'}^t$
   \IF{$\Omega^A(\overline{\alpha_A}, \overline{\alpha_P}, \overline{\theta}) < \Omega^P(\overline{\alpha_A}', \overline{\alpha_P}', \overline{\theta}')$}
    \STATE Return $(\overline{\alpha_A}, \overline{\alpha_P}, \overline{\theta})$
    \ELSE
    \STATE Return $(\overline{\alpha_A}', \overline{\alpha_P}', \overline{\theta}')$
   \ENDIF
\end{algorithmic}
\end{algorithm}

% Similarly as before, write
% \begin{align*}
%     \Omega(\alpha_A, \alpha_P, \theta) &= (\alpha_A r_A + \alpha_P r_P) + \frac{1}{n_A n_P} \sum_{i=1}^{n_A} \sum_{j=1}^{n_P} (f(y^P_j\langle\theta, (\hat{x}_{i,j}, a^A_i) \rangle) + \max(y^P_j\langle\theta, (\hat{x}_{i,j}, a^A_i) \rangle-\alpha_P\kappa_P,0) -\hat{\alpha}||x^A_i -x^P_j||)
% \end{align*} where the values of $(\hat{x}_{i,j}, \hat{\alpha})$ are determined by $(\alpha_A, \alpha_P)$ given to $\Omega$. 

\drjoinalgconvergence*
\begin{proof}
    Note that due to convergence rate of projected gradient descent, we have 
    \begin{align*}
    \Omega^A(\overline{\alpha_A}, \overline{\alpha_P}, \overline{\theta}) &\le \Omega^A(\alpha'_A, \alpha'_P, \theta') + O\left(\frac{1}{\sqrt{T}}\right) \\
    \Omega^P(\overline{\alpha_A}', \overline{\alpha_P}', \overline{\theta}') &\le \Omega^P(\alpha{''}_A, \alpha{''}_P, \theta{''}) + O\left(\frac{1}{\sqrt{T}}\right)
\end{align*}
Also, we have
\begin{align*}
    \Omega^A(\overline{\alpha_A}, \overline{\alpha_P}, \overline{\theta}) &= \Omega(\overline{\alpha_A}, \overline{\alpha_P}, \overline{\theta}) \\ 
    \Omega^A(\alpha'_A, \alpha'_P, \theta') &= \Omega(\alpha'_A, \alpha'_P, \theta') \\
    \Omega^P(\overline{\alpha_A}', \overline{\alpha_P}', \overline{\theta}') &= \Omega(\overline{\alpha_A}', \overline{\alpha_P}', \overline{\theta}') \\
    \Omega^P(\alpha{''}_A, \alpha{''}_P, \theta{''}) &= \Omega(\alpha{''}_A, \alpha{''}_P, \theta{''})
\end{align*}

Therefore, we must have 
\begin{align*}
    \Omega(\alpha_A, \alpha_P, \theta) \le \Omega(\alpha^*_A, \alpha^*_P, \theta^*) + O\left(\frac{1}{\sqrt{T}}\right)
\end{align*}
\end{proof}

Here we try to give a characterization of the projection when $p=2$. It is not immediate clear how to perform a projection onto $C$: given $\theta, \alpha_A, \alpha_P$, we need to find
\[
    \arg\min_{\theta', \alpha'_A, \alpha_P' \in C_1} ||(\theta, \alpha_A, \alpha_P) -  (\theta', \alpha'_A, \alpha'_P)||^2_2 =\arg\min_{\theta', \alpha'_A, \alpha_P' \in C_1} || \theta - \theta'||^2_2 + |\alpha_A - \alpha'_A|^2 + |\alpha_P - \alpha'_P|^2.
\]

Suppose we are given $(\theta, \alpha_A, \alpha_P)$ such that $||\theta_1||_2 > \alpha_A + \alpha_P$ and/or $||\theta_2||_2 > \kappa_A\alpha_A$. The Lagrangian for the above optimization problem we are interested in is the following:
\begin{align*}
    &\Lagr(\theta'_1, \theta'_2, \alpha'_A, \alpha'_P) \\
    &= \frac{1}{2} \sum_i (\theta'_1[i] - \theta_1[i])^2 + \frac{1}{2} \sum_i (\theta'_2[i] - \theta_2[i])^2 + \frac{1}{2} (\alpha'_A - \alpha_A)^2 + \frac{1}{2} (\alpha'_P - \alpha_P)^2\\
    &+ \lambda_1 ((\sum_i (\theta'_1[i])^2)^{1/2} - \alpha'_A - \alpha'_P)
    + \lambda_2 ((\sum_i (\theta'_2[i])^2)^{1/2} - \kappa_A \alpha'_A ) 
    + \lambda_3 (\alpha'_A - \alpha'_P).
\end{align*}

The stationary part of the KKT condition requires that the gradient with respect to $\theta'_1, \theta'_2, \alpha'_A$ and $\alpha'_P$ is 0. In other words, we have
\begin{align*}
    \nabla_{\theta'_1[i]} \Lagr &=  (\theta'_1[i] - \theta_1[i]) + \frac{\lambda_1}{2} \frac{2\theta'_1[i]}{(\sum_i (\theta'_1[i])^2)^{1/2}} = (\theta'_1[i] - \theta_1[i]) +  \frac{\lambda_1\theta'_1[i]}{||\theta'_1||_2} = 0 \\
    \nabla_{\theta'_2[i]} \Lagr &=  (\theta'_2[i] - \theta_2[i]) + \frac{\lambda_1}{2} \frac{2\theta'_2[i]}{(\sum_i (\theta'_2[i])^2)^{1/2}} = (\theta'_2[i] - \theta_2[i]) +  \frac{\lambda_2\theta'_1[i]}{||\theta'_2||_2} = 0 \\
    \nabla_{\alpha'_A} \Lagr &= \alpha'_A - \alpha_A -\lambda_1 -\lambda_2\kappa_A  +\lambda_3= 0 \\
    \nabla_{\alpha'_P} \Lagr &= \alpha'_P - \alpha_P -\lambda_1 - \lambda_3= 0 \\
\end{align*}

With some arranging, we get
\begin{align*}
    &\theta'_1  +  \frac{\lambda_1\theta'_1}{||\theta'_1||} = \theta_1 \\
    \implies &||\theta'_1|| \bar{\theta'_1} + \lambda_1 \bar{\theta'_1} = \theta_1 \\
    \implies &\bar{\theta'_1} = \frac{\theta_1}{||\theta'_1||  + \lambda_1}\\
    \implies &\theta'_1 = \frac{||\theta'_1|| \theta_1}{||\theta'_1||  + \lambda_1} \\
    \implies &||\theta'_1|| = \norm{\frac{||\theta'_1|| \theta_1}{||\theta'_1||  + \lambda_1}}\\
    \implies &||\theta'_1|| = \frac{||\theta'_1||}{||\theta'_1||  + \lambda_1} ||\theta_1||\\
    \implies &||\theta'_1||  + \lambda_1  = ||\theta_1||
\end{align*}

Similarly, we have
\begin{align*}
    &\theta'_2  +  \frac{\lambda_2\theta'_2}{||\theta'_2||} = \theta_2 \\
    \implies &||\theta'_2|| \bar{\theta'_2} + \lambda_2 \bar{\theta'_2} = \theta_1 \\
    \implies &\bar{\theta'_2} = \frac{\theta_2}{||\theta'_2||  + \lambda_2}\\
    \implies &\theta'_2 = \frac{||\theta'_2|| \theta_2}{||\theta'_2||  + \lambda_2} \\
    \implies & ||\theta'_2|| = ||\frac{||\theta'_2|| \theta_2}{||\theta'_2||  + \lambda_2}|| \\
    \implies & ||\theta'_2|| = \frac{||\theta'_2|| }{||\theta'_2||  + \lambda_2} ||\theta_2|| \\
    \implies & ||\theta'_2||  + \lambda_2 = ||\theta_2||.
\end{align*}

Note that $\theta'_1$ is simply a rescaling of $\theta_1$:
\[
    \theta'_1 = \frac{||\theta_1|| - \lambda_1}{||\theta_1||} \theta_1.
\]
The complementary slack conditions require that
\[
    \lambda_1 (||\theta'_1|| - \alpha'_A - \alpha'_P) = 0.
\]
In other words, either $\theta'_1 = \theta_1$ or $||\theta'_1|| = \alpha'_A + \alpha'_P$. 
The same argument applies for $\theta'_2$: either $\theta'_2 = \theta_2$ or $||\theta'_2|| = \kappa_A \alpha_A$. Now, we consider all four cases, and for each of those cases, we repeatedly consider the case where $\lambda_3 = 0$ and $\lambda_3 > 0$ (i.e. $\alpha'_A - \alpha'_P = 0$ from the complementary slack condition).

\paragraph{Case $\theta'_1 = \theta_1$ and $\theta'_2 = \theta_2$:} 
In this case, we need only concern ourselves with how to set $\alpha'_A$ and $\alpha'_P$. Because we have $\lambda_1, \lambda_2 = 0$, 
\begin{align*}
    \alpha'_A - \alpha_A +  \lambda_3 &= 0\\
    \lambda_3 &= \alpha'_P - \alpha_P.
\end{align*}

The complementary slackness condition requires
$\lambda_3 (\alpha'_A - \alpha'_P) = 0$. In other words, when $\lambda_3 = 0$, we have $(\alpha'_A, \alpha'_P) = (\alpha_A, \alpha_P)$. In other case where $\alpha'_A = \alpha'_P$, we have $(\alpha'_A, \alpha'_P) = (\frac{\alpha_A + \alpha_P}{2}, \frac{\alpha_A + \alpha_P}{2})$.

\paragraph{Case $\theta'_1 = \theta_1$ and
$||\theta'_2|| = \kappa_A \alpha'_A$:} We have $\lambda_1 = 0$ and 
\[
    \lambda_2 = ||\theta_2|| - ||\theta'_2|| = ||\theta_2|| - \kappa_A \alpha'_A
\]

Plugging in $\lambda_1 = 0$, we have
\begin{align*}
    \alpha'_A - \alpha_A  -\lambda_2\kappa_A  +\lambda_3= 0 \\
    \alpha'_P - \alpha_P - \lambda_3= 0.
\end{align*}

Substituting in $\lambda_2$ value, we get
\begin{align*}
    &\alpha'_A - \alpha_A  - \kappa_A(||\theta_2|| - \kappa_A \alpha'_A)  +\lambda_3= 0 \\
    \implies& \alpha'_A (1 + \kappa^2_A) = \alpha_A + \kappa_A ||\theta_2|| - \lambda_3 \\
    &\implies \alpha'_A  = \frac{\alpha_A + \kappa_A ||\theta_2|| - \lambda_3}{1 + \kappa^2_A}
\end{align*}

If $\lambda_3 = 0$, we have 
\begin{align*}
    \alpha'_A &=  \frac{\alpha_A + \kappa_A ||\theta_2||}{1 + \kappa^2_A}\\
    \alpha'_P &=  \alpha_P.
\end{align*}

If $\lambda_3 \neq 0$ and hence $\alpha'_A = \alpha'_P$, then we have
\begin{align*}
    &\alpha'_A (1 + \kappa^2_A) = \alpha_A + \kappa_A ||\theta_2|| - (\alpha'_A -\alpha_P) \\
    \implies &\alpha'_P = \alpha'_A = \frac{\alpha_A + \alpha_P +  \kappa_A ||\theta_2||}{2 + \kappa^2_A}
\end{align*}

\paragraph{Case $||\theta'_1|| = \alpha'_A + \alpha'_P$ and $\theta'_2 = \theta_2$:}
We have that $\lambda_2 = 0$ and $\lambda_1 > 0$ and also
\begin{align*}
    \lambda_1  = ||\theta_1|| - ||\theta'_1|| = ||\theta_1|| - (\alpha'_A + \alpha'_P).
\end{align*}

Plugging in $\lambda_2 = 0$, we have
\begin{align*}
    &\alpha'_A - \alpha_A -\lambda_1 +\lambda_3= 0 \\
    &\alpha'_P - \alpha_P -\lambda_1 - \lambda_3= 0 \\
\end{align*}

If $\lambda_3 = 0$, then 
\begin{align*}
    &\alpha'_A - \alpha_A -\lambda_1 = 0 \quad\text{and}\quad\alpha'_P - \alpha_P -\lambda_1 = 0 \\
    \implies & \lambda_1 = \alpha'_A - \alpha_A = \alpha'_P - \alpha_P
\end{align*}

Substituting $\alpha'_A = \alpha_A + \alpha'_P - \alpha_P$ into $\alpha'_P - \alpha_P = \lambda_1  = ||\theta_1|| - (\alpha'_A + \alpha'_P)$, we get
\begin{align*}
    &\alpha'_P - \alpha_P = ||\theta_1|| - (\alpha_A + 2\alpha'_P - \alpha_P) \\
    \implies& \alpha'_P - \alpha_P = ||\theta_1|| - \alpha_A - 2\alpha'_P + \alpha_P \\
    \implies& 3\alpha'_P  = ||\theta_1|| - \alpha_A  + 2\alpha_P \\
    \implies& \alpha'_P  = \frac{||\theta_1|| - \alpha_A  + 2\alpha_P }{3}.
\end{align*}
$\alpha'_A$ is then calculated as 
\begin{align*}
    \alpha'_A = \alpha_A + \frac{||\theta_1|| - \alpha_A  + 2\alpha_P }{3} - \alpha_P.
\end{align*}

If $\lambda_3 \neq 0$ and hence $\alpha'_A = \alpha'_P$, then
\begin{align*}
    ||\theta_1|| - 2\alpha'_A = \alpha'_A - \alpha_A + \lambda_3 = \alpha'_A - \alpha_P  - \lambda_3 = \lambda_1
\end{align*}

From the first equation, we get
\[
\lambda_3 = ||\theta_1|| - 2\alpha'_A - (\alpha'_A - \alpha_A) = ||\theta_1|| - 3\alpha'_A + \alpha_A.
\] 

Plugging in this value for $\lambda_3$ into the second equation, we get
\begin{align*}
    &||\theta_1|| - 2\alpha'_A = \alpha'_A - \alpha_P - (||\theta_1|| - 3\alpha'_A + \alpha_A) \\
    \implies & - 2\alpha'_A = 4\alpha'_A - \alpha_P - \alpha_A - 2 ||\theta_1|| \\
    \implies & \frac{\alpha_A + \alpha_P + 2 ||\theta_1||}{6} = \alpha'_A.
\end{align*}

\paragraph{Case $||\theta'_1|| = \alpha'_A + \alpha'_P$ and $||\theta'_2|| = \kappa_A \alpha'_A$:}

\begin{align*}
    \lambda_1 & = ||\theta_1|| - ||\theta'_1|| = ||\theta_1|| - (\alpha'_A + \alpha'_P) \\
    \lambda_2 &= ||\theta_2|| - ||\theta'_2|| = ||\theta_2|| - \kappa_A \alpha'_A
\end{align*}

Putting these equations altogether with variables $\alpha'_A, \alpha'_P, \lambda_1, \lambda_2, \lambda_3$, we have
\begin{align*}
    ||\theta_1|| &= \alpha'_A + \alpha'_P  + \lambda_1 \\
    ||\theta_2|| &= \kappa_A \alpha'_A + \lambda_2\\
    \alpha'_A - \alpha_A  -\lambda_1 -\lambda_2\kappa_A  +\lambda_3 &= 0 \\
    \alpha'_P - \alpha_P - \lambda_1 - \lambda_3 &= 0
\end{align*}

We'll use the first equation to substitute in $\lambda_1 = ||\theta_1|| -\alpha'_A - \alpha'_P$ to get
\begin{align*}
    ||\theta_2|| &= \kappa_A \alpha'_A + \lambda_2\\
    2\alpha'_A - \alpha_A  - ||\theta_1||  + \alpha'_P -\lambda_2\kappa_A  +\lambda_3 &= 0 \\
    2\alpha'_P - \alpha_P - ||\theta_1|| +\alpha'_A - \lambda_3 &= 0
\end{align*}

Similarly, use the last equation to substitute in $\lambda_3 = 2\alpha'_P - \alpha_P - ||\theta_1|| +\alpha'_A$. 
\begin{align*}
    ||\theta_2|| &= \kappa_A \alpha'_A + \lambda_2\\
    3\alpha'_A + 3\alpha'_P - \alpha_A  - \alpha_P - 2||\theta_1||  -\lambda_2\kappa_A  &= 0 
\end{align*}

Finally, plug in $\lambda_2 = ||\theta_2|| - \kappa_A \alpha'_A$.
\begin{align*}
    &3\alpha'_A + 3\alpha'_P - \alpha_A  - \alpha_P - 2||\theta_1|| - \kappa_A(||\theta_2|| - \kappa_A \alpha'_A) = 0 \\
    \implies &(3 + \kappa^2_A)\alpha'_A + 3\alpha'_P - \alpha_A  - \alpha_P - 2||\theta_1|| - \kappa_A||\theta_2|| = 0 \\
\end{align*}

As before, when $\lambda_3 = 0$, we get 
\begin{align*}
    &2\alpha'_P - \alpha_P - ||\theta_1|| +\alpha'_A = 0\\
    \implies & \alpha'_P = \frac{||\theta_1|| + \alpha_P - \alpha'_A}{2}.
\end{align*}

Then, we get 
\begin{align*}
    &(3 + \kappa^2_A)\alpha'_A + 3\left(\frac{||\theta_1|| + \alpha_P - \alpha'_A}{2}\right) - \alpha_A  - \alpha_P - 2||\theta_1|| - \kappa_A||\theta_2|| = 0\\
    \implies& \left(\frac{3}{2} + \kappa^2_A\right)\alpha'_A = -3\left(\frac{||\theta_1|| + \alpha_P}{2}\right) + \alpha_A  + \alpha_P + 2||\theta_1|| + \kappa_A||\theta_2||\\
    \implies& \alpha'_A = \frac{-3\left(\frac{||\theta_1|| + \alpha_P}{2}\right) + \alpha_A  + \alpha_P + 2||\theta_1|| + \kappa_A||\theta_2||}{\frac{3}{2} + \kappa^2_A}\\
    \implies& \alpha'_A = \frac{-3\left(\frac{||\theta_1|| + \alpha_P}{2}\right) + \alpha_A  + \alpha_P + 2||\theta_1|| + \kappa_A||\theta_2||}{\frac{3}{2} + \kappa^2_A}\\
    \implies& \alpha'_A = \frac{2\alpha_A + \alpha_P + ||\theta_1|| + 2 \kappa_A ||\theta_2||}{3+2\kappa^2_A}.
\end{align*}
Consequently, we have
\begin{align*}
    \alpha'_P = \frac{||\theta_1|| + \alpha_P}{2} -  \left(\frac{2\alpha_A + \alpha_P + ||\theta_1|| + 2 \kappa_A ||\theta_2||}{6+4\kappa^2_A}\right).
\end{align*}

Otherwise, when $\lambda_3 > 0$, we have $\alpha'_A = \alpha'_P$. In this case, we get
\begin{align*}
    &(6 + \kappa^2_A)\alpha'_A = \alpha_A  + \alpha_P + 2||\theta_1|| + \kappa_A||\theta_2|| \\
    \implies& \alpha'_A = \alpha'_P = \frac{\alpha_A  + \alpha_P + 2||\theta_1|| + \kappa_A||\theta_2||}{6 + \kappa^2_A}.
\end{align*}

We summarize the results in the following tables:
\begin{table}[ht]
\begin{tabular}{|l|l|l|}
\hline
Cases & $\lambda_3 = 0$ \\ \hline
$(\theta'_1, \theta'_2) = (\theta_1, \theta_2)$ & $(\alpha'_A, \alpha'_P) = (\alpha_A, \alpha_P)$  \\ \hline
$(\theta'_1, \theta'_2) = (\theta_1, \kappa_A \alpha'_A \overline{\theta}_2)$ & $(\alpha'_A, \alpha'_P)= (\frac{\alpha_A + \kappa_A ||\theta_2||}{1 + \kappa^2_A}, \alpha_P)$  \\ \hline
$(\theta'_1, \theta'_2) = ((\alpha'_A + \alpha'_P)\overline{\theta}_1, \theta_2)$  & $(\alpha'_A, \alpha'_P) = (\alpha_A + \alpha'_P - \alpha_P, \frac{||\theta_1|| - \alpha_A  + 2\alpha_P }{3}) $ \\ \hline
$(\theta'_1, \theta'_2) = ((\alpha'_A + \alpha'_P)\overline{\theta}_1, \kappa_A \alpha'_A \overline{\theta}_2)$ & $(\alpha'_A, \alpha'_P) = \left(\frac{2\alpha_A + \alpha_P + ||\theta_1|| + 2 \kappa_A ||\theta_2||}{3+2\kappa^2_A}, \frac{||\theta_1|| + \alpha_P}{2} -  \left(\frac{2\alpha_A + \alpha_P + ||\theta_1|| + 2 \kappa_A ||\theta_2||}{6+4\kappa^2_A}\right) \right)$   \\ \hline
\end{tabular}
\end{table}

\begin{table}[ht]
\begin{tabular}{|l|l|}
\hline
Cases &  $\lambda_3 > 0$\\ \hline
$(\theta'_1, \theta'_2) = (\theta_1, \theta_2)$ & $\alpha'_A =\alpha'_P = \frac{\alpha_A + \alpha_P}{2}$  \\ \hline
$(\theta'_1, \theta'_2) = (\theta_1, \kappa_A \alpha'_A \overline{\theta}_2)$   & $\alpha'_A =\alpha'_P=\frac{\alpha_A + \alpha_P +  \kappa_A ||\theta_2||}{2 + \kappa^2_A}$ \\ \hline
$(\theta'_1, \theta'_2) = ((\alpha'_A + \alpha'_P)\overline{\theta}_1, \theta_2)$  & $\alpha'_A = \alpha'_P = \frac{\alpha_A + \alpha_P + 2 ||\theta_1||}{6}$  \\ \hline
$(\theta'_1, \theta'_2) = ((\alpha'_A + \alpha'_P)\overline{\theta}_1, \kappa_A \alpha'_A \overline{\theta}_2)$ & $\alpha'_A = \alpha'_P = \frac{\alpha_A  + \alpha_P + 2||\theta_1|| + \kappa_A||\theta_2||}{6 + \kappa^2_A}$ \\ \hline
\end{tabular}
\end{table}

\section{Application: Fairness}
\label{app:fairness}
In many situations, the actual demographic group information may not be available in the original labeled dataset, but another auxiliary unlabeled dataset may contain the needed demographic group information. We can leverage our data join method in order to incorporate this auxiliary dataset to penalize the model for model's unfairness. Suppose $\cA$ represents two different groups that an individual can belong to --- $\cA = \{0, 1\}$.

Given $\theta$, we define its unfairness with respect to distribution $\cP$ over $\cX, \cA, \cY$ as 
$\cU(\theta, \cP)= \Bigg|\Pr_{(x,a,y) \sim \cP}\left[u(h_{\theta}(x))| a=1, y=1\right] - \Pr_{(x,a,y) \sim \cP}\left[u(h_{\theta}(x))| a=0, y=1\right]\Bigg|$
where $u(t) = \log(t)$ and $h_{\theta}(x) = \frac{1}{1 + \exp(-\langle \theta, x \rangle)}$ as in \cite{drofair}. This term is similar to the difference in true positive rates as in the case of equal opportunity, but it differs in that it looks at the log-probability --- this fairness criterion is referred to as log-probabilistic equalized opportunity in \cite{drofair}.

Also, as in \cite{drofair}, we suppose we know the underlying positive rates for each group and constrain the joint distribution's marginal distribution over $\cA$ and $\cY$ in the following manner: given some $p_0, p_1 \in (0,1)$, we define $W_{(p_0, p_1)}(S_A, S_P, r_A, r_P) = \{\cQ \in W(S_A, S_P, r_A, r_P): \Pr_{(x,a,y) \sim \cQ}[a=0, y=1] = p_0, \Pr_{(x,a,y) \sim \cQ}[a=0, y=1] = p_1\}.$ 

Then, the problem we are interested in is
$\min_{\theta \in \Theta} \sup_{\cQ \in W_{(p_0, p_1)}(S_A, S_P, r_A, r_P)} \\
    \E_{(x,a,y) \sim Q}[\ell(\theta, (x,a,y))] + \eta\cU(\theta, \cQ)
\inlineeqnum\label{eqn:disjoint-learning-fair}$
where we are adding a fairness regularization term multiplied by some constant $\eta$. In Appendix \ref{app:fairness}, just as in Section \eqref{sec:tractable-optimization}, we show how to write down the optimization by making the coupling explicit (problem \eqref{eqn:disjoint-learning-transport-fair}), derive the dual (problem \eqref{eqn:transport-dual-fair}), and how to approximate the supremum term that appears in the dual problem (Lemma \ref{lem:conjugate-infimal-repr-fair}) to derive our final fair distributionally robust data join problem (problem \eqref{eqn:fairness-prob-final} which is basically a fairness regularized version of \eqref{eqn:final-opt-prob}) with a similar approximation guarantee (Theorem \eqref{thm:gap-sup-fairness}).

We remark that solving for $\sup_{\cQ \in W_{(p_0, p_1)}(S_A, S_P, r_A, r_P)} \cU(\theta, \cQ)$ for some fixed $\theta$, which can be indeed solved with minimal modifications, corresponds to estimating the worst case unfairness of $\theta$ over all distributions $\cQ \in W_{(p_0, p_1)}(S_A, S_P, r_A, r_P)$. \cite{kallus2021assessing} consider a special case where $r_A, r_P = 0, 0$, but they can handle various fairness measures.

\subsection{Analysis}
Following the same argument as in Section \ref{subsec:coupling-formulation}, consider the following problem with some fixed $p_0, p_1 \in (0,1)$ and $|\eta| < \min(p_0, p_1)$.
\begin{align}
\begin{split}
     \sup_{\pi^{a,y}_{i,j}}  &\sum_{i=1}^{n_A} \sum_{j=1}^{n_P} \sum_{a \in A} \sum_{y \in \cY}\int \left(\ell(\theta, (x, a, y)) + \eta \left(u(h_{\theta}((x,a)))\frac{\ind[a=1, y=1]}{p_1} - u(h_{\theta}((x,a)))\frac{\ind[a=0, y=1] }{p_0} \right) \right) \pi^{a,y}_{i,j}(dx) \\
     \text{s.t.} \quad& \sum_{i=1}^{n_A}\sum_{j=1}^{n_P} \sum_{a \in A} \sum_{y \in \cY}\int d_A^{i}(x,a)  \pi^{a,y}_{i,j}(dx) \le r_a\\
    \quad& \sum_{i=1}^{n_A}\sum_{j=1}^{n_P} \sum_{a \in A} \sum_{y \in \cY}\int d_P^{j}(x,y)  \pi^{a,y}_{i,j}(dx) \le r_P \\
    \quad& \sum_{j=1}^{n_P}\sum_{a \in A} \sum_{y \in \cY} \int \pi^{a,y}_{i,j}(dx) = \frac{1}{n_A} \quad \forall i \in [n_A]\\
    \quad& \sum_{i=1}^{n_A}\sum_{a \in A} \sum_{y \in \cY} \int  \pi^{a,y}_{i,j}(dx)  = \frac{1}{n_P} \quad \forall j \in [n_P]\\
    \quad& \sum_{i=1}^{n_A}\sum_{a \in A} \sum_{y \in \cY} \int \ind[a=0, y=1] \pi^{a,y}_{i,j}(dx)  = p_{0} \\
    \quad& \sum_{i=1}^{n_A}\sum_{a \in A} \sum_{y \in \cY} \int \ind[a=1, y=1] \pi^{a,y}_{i,j}(dx)  = p_{1}
\end{split}
\label{eqn:disjoint-learning-transport-fair}
\end{align}

Denoting the value of the above optimization as $\pfair(p_0, p_1, \eta)$, the same argument as in Theorem \ref{thm:coupling-formulation} can be used to see that the value of \eqref{eqn:disjoint-learning-fair} is exactly $\max(\pfair(p_0, p_1, \eta), \pfair(p_0, p_1, -\eta))$ where we need to try out $\eta$ and $-\eta$ in order to handle the absolute value in $\cU$.

As in Section \ref{subsec:duality}, the dual problem of \eqref{eqn:disjoint-learning-transport-fair} can be derived by looking at the Lagrangian, which after rearranging the terms a little bit is as follows:
\begin{align*}
    &\Lagr(\pi, \alpha_A, \alpha_P, \{\beta_{i}\}, \{\beta'_{j}\}\}, \gamma_0, \gamma_1)\\
    &= \sum_{i=1}^{n_A}\sum_{j=1}^{n_P} \sum_{a \in \cA} \sum_{y \in \cY}\int \Bigg(\ell(\theta, (x, a, y)) + \eta \left(u(h_{\theta}((x,a)))\frac{\ind[a=1, y=1]}{p_1} - u(h_{\theta}((x,a)))\frac{\ind[a=0, y=1] }{p_0} \right) \\
    &- \alpha_A d^i_A(x,a) -  \alpha_P d^j_P(x,y) - \beta_i  - \beta'_j \\
    &- \gamma_0 \ind[a=0, y=1] - \gamma_1 \ind[a=1, y=1]\Bigg) \pi^{y, a}_{i,j}(dx) \\
    &+ \alpha_A r_A + \alpha_P r_P + \frac{1}{n_A} \sum_{i=1}^{n_A} \beta_i + \frac{1}{n_P} \sum_{j=1}^{n_P} \beta'_j + p_0 \gamma_0 + p_1 \gamma_1. 
\end{align*}

The dual problem is then
\begin{align}
\begin{split}
    &\inf_{\substack{\alpha_A, \alpha_P, \\ \{\beta_{i}\}_{i \in [n_A]}, \{\beta'_j\}_{j \in [n_P]}}} \quad \alpha_A r_A + \alpha_P r_P + \frac{1}{n_A}\sum_{i\in [n_A]} \beta_{i} + \frac{1}{n_P} \sum_{j \in [n_P]} \beta'_j + p_0 \gamma_0 + p_1 \gamma_1\\
    &\text{s.t.} \sup_{x} \Bigg(\ell(\theta, (x, a, y)) + \eta u(h_{\theta}((x,a))) \left(\frac{\ind[a=1, y=1]}{p_1} - \frac{\ind[a=0, y=1] }{p_0} \right) \\
    &- \alpha_A d^i_A(x,a) -  \alpha_P d^j_P(x,y)\Bigg) - \beta_{i} - \beta'_j \\
    &- \gamma_0\ind[a=0, y=1] - \gamma_1\ind[a=1,y=1] \le 0 \quad i \in [n_A], j \in [n_P], a \in \cA, y \in \cY 
\end{split}
\label{eqn:transport-dual-fair}
\end{align}

Note that when $y=-1$, then the term in the supremum simply becomes
\begin{align*}
    \ell(\theta, (x, a, y)) - \alpha_A d^i_A(x,a) -  \alpha_P d^j_P(x,y).
\end{align*}

When $y=1$, then we get
\begin{align*}
    &\log(1+\exp(- \langle \theta, (x,a) \rangle)) -\eta \log(1+\exp(- \langle \theta, (x,a) \rangle) \left(\frac{\ind[a=1]}{p_1} - \frac{\ind[a=0] }{p_0} \right) - \alpha_A d^i_A(x,a) -  \alpha_P d^j_P(x,y)\\
    &=\log(1+\exp(- \langle \theta, (x,a) \rangle)) \left(1-\eta \left(\frac{\ind[a=1]}{p_1} - \frac{\ind[a=0] }{p_0} \right) \right) - \alpha_A d^i_A(x,a) -  \alpha_P d^j_P(x,y)\\
    &=\left(1 - \eta \left(\frac{\ind[a=1]}{p_1} - \frac{\ind[a=0] }{p_0} \right) \right)\ell(\theta, (x,a,y)) - \alpha_A d^i_A(x,a) -  \alpha_P d^j_P(x,y)
\end{align*}

For simplicity, we write
\[
    c(a,y) = 1 - \ind[y=1] \eta\left(\frac{\ind[a=1]}{p_1} - \frac{\ind[a=0] }{p_0} \right).
\]
Note that $c(a,y) > 0$ because $\eta > \min(p_0, p_1)$. Write $\overline{c} = \max_{a,y} c(a,y)$.

Then, the above dual problem can be re-written as 
\begin{align}
\begin{split}
    \inf_{\substack{\alpha_A, \alpha_P, \\ \{\beta_{i}\}_{i \in [n_A]}, \{\beta'_j\}_{j \in [n_P]}}} \quad &\alpha_A r_A + \alpha_P r_P + \frac{1}{n_A}\sum_{i\in [n_A]} \beta_{i} + \frac{1}{n_P} \sum_{j \in [n_P]} \beta' + p_0 \gamma_0 + p_1 \gamma_1\\
    \text{s.t.}\quad& \sup_{x} \Bigg(c(a, y) \cdot \ell(\theta, (x, a, y))  - \alpha_A d^i_A(x,a) -  \alpha_P d^j_P(x,y)\Bigg) - \beta_{i} - \beta'_j \\
    &- \gamma_0\ind[a=0, y=1] - \gamma_1\ind[a=1,y=1] \le 0 \quad i \in [n_A], j \in [n_P], a \in \cA, y \in \cY.
\end{split}
\label{eqn:transport-dual-fair2}
\end{align}

We remark that the $c(a,y)$ that folds the fairness constraint into the original loss is essentially equivalent to the cost plugged into the cost-sensitive oracle in \cite{agarwal2018reductions} and \cite{kearns2018preventing}. 

Note that the constant can be taken out of the $\sup$ as $c(a,y)$ is always positive and the same proof as in Lemma \ref{lem:conjugate-infimal-repr} can be used:
\begin{align}
    &\sup_{x} \left(c(a, y) \cdot \ell(\theta, (x, a, y))  - \alpha_A d^i_A(x,a) -  \alpha_P d^j_P(x,y)\right) \nonumber\\
    &= c(a,y) \cdot \sup_{x} \left(\ell(\theta, (x, a, y))  - \frac{\alpha_A}{c(a,y)} d^i_A(x,a) -  \frac{\alpha_P}{c(a,y)} d^j_P(x,y)\right)\label{eqn:fairness-c-taken-out}
\end{align}

We have intentionally taken $a$ outside the $\sup$ so that we can view $c(a,y)$ as a constant.

\begin{lemma}
\label{lem:conjugate-infimal-repr-fair}
Fix any $\theta$, $(x^A_i, a, x^P_j)$, and $(\alpha_A, \alpha_P, \kappa_A)$. If $||\theta[1:m_1]||_{p,*} > \alpha_A + \alpha_P$, then $\sup_{x} h(\theta, (x, a)) - \alpha_A ||x^A_i - x ||_{p} -  \alpha_P||x^P_j - x ||_{p} = \infty$.
Otherwise, we have 
\begin{align*}
    &\sup_{x} h(\theta, (x, a)) - \alpha_A ||x^A_i - x ||_{p} -  \alpha_P||x^P_j - x ||_{p}\\
    &=\sup_{b \in [0,1]} - f^*(b) + (g^i_1 \square g^j_2)(-b\theta[1:m_1]) + \langle b\theta[m_1+1:m_1+m_2], a \rangle
\end{align*}
where $g^i_1$ and $g^j_2$ is the same as defined in Lemma \ref{lem:conjugate-infimal-repr}.
\end{lemma}
\begin{proof}
Noting that $h$ is convex and thus $h$ is equal to its biconjugate $h^{**}$, we have
\begin{align*}
    &\sup_{x} h(\theta, (x, a)) - \alpha_A ||x^A_i - x ||_{p} -  \alpha_P||x^P_j - x ||_{p} \\
    &=\sup_{x} \sup_{b \in [0,1]} \langle b\theta, (x,a) \rangle  - f^*(b) - \alpha_A ||x^A_i - x||_{p}  -  \alpha_P ||x^P_j - x||_{p} \\
    &= \sup_{b \in [0,1]} \sup_{x} \langle b\theta, (x,a) \rangle  - f^*(b) - \sup_{||q_1||_{p,*} \le \alpha_A}\langle q_1, x^A_i - x \rangle -  \sup_{||q_2||_{p,*} \le \alpha_P}\langle q_2, x^P_j - x\rangle \\ 
    &=  \sup_{b \in [0,1]} \sup_{x} \langle b\theta, (x,a) \rangle  - f^*(b) - \sup_{||q_1||_{p,*} \le \alpha_A}\langle (q_1, \xi), (x^A_i,a) - (x,a) \rangle -  \sup_{||q_2||_{p,*} \le \alpha_P}\langle (q_2, 0), (x^P_j,a) - (x,a)\rangle \\
    &=  \sup_{b \in [0,1]} \sup_{x} \inf_{\substack{||q_1||_{p,*}\le \alpha_A,\\ ||q_2||_{p,*}\le \alpha_P}} \langle b\theta, (x,a) \rangle  - f^*(b)- \langle (q_1, \xi), (x^A_i,a) - (x,a) \rangle - \langle (q_2, 0), (x^P_j,a) - (x,a)\rangle \\
    &= \sup_{b \in [0,1]} \sup_{x}  \inf_{\substack{||q_1||_{p,*}\le \alpha_A,\\ ||q_2||_{p,*}\le \alpha_P}} \langle b\theta + (q_1, \xi) + (q_2,0) , (x,a) \rangle  - f^*(b) -\langle q_1, x^A_i \rangle - \langle q_2, x^P_j \rangle
\end{align*}   
where $\xi$ can be chosen arbitrarily.

Appealing to proposition 5.5.4 of \cite{bertsekas2009convex}, we can swap $\inf$ and $\sup$.
\begin{align*}
    &= \sup_{b \in [0,1]}  \inf_{\substack{||q_1||_{p,*}\le \alpha_A,\\ ||q_2||_{p,*}\le \alpha_P }} \sup_{x}  \langle b\theta + (q_1, \xi) + (q_2,0), (x,a) \rangle  - f^*(b) - \langle q_1, x^A_i \rangle - \langle q_2, x^P_j\rangle.
\end{align*}
Note that unless $b\theta + (q_1, \xi) + (q_2,0) = 0$, $x$ can be chosen arbitrarily big. Also, if $\theta + (q_1, \xi) + (q_2, 0)\neq 0$, then $b$ can be chosen to be 1. Therefore, if there doesn't exist $(q_1, q_2)$ such that $\theta + (q_1, \xi) + (q_2,0) = 0$, everything evaluates to $\infty$. 
In other words, the expression evaluates to $\infty$ unless $||\theta[1:m_1]||_{p,*} \le \alpha_A + \alpha_P$ and $\xi = \theta[m_1+1:m_1+m_2]$.

Now, suppose $\theta$ satisfies the above condition as we know it evaluates to $\infty$ otherwise. Then, we get
\begin{align*}
    &= \sup_{b \in [0,1]} - f^*(b) + \langle b\theta[m_1 +1: m_1 + m_2], a \rangle + \inf_{\substack{||q_1||_{p,*}\le \alpha_A,\\ ||q_2||_{p,*}\le \alpha_P}} \begin{cases}
          - \langle q_1, x^A_i \rangle - \langle q_2, x^P_j \rangle  &\quad\text{if $b\theta[1:m_1] + q_1 + q_2 = 0$} \\
         \infty &\quad\text{otherwise}
    \end{cases}\\
    &=\sup_{b \in [0,1]} - f^*(b) + (g^i_1 \square g^j_2)(-b\theta[1:m_1]) + \langle b\theta[m_1 +1: m_1 + m_2], a \rangle.
\end{align*}
\end{proof}

\begin{restatable}{theorem}{thmsupupperfair}
\label{thm:sup_upper_fair}
Fix any $\theta$, $(x^A_i, a, x^P_j)$, and $(\alpha_A, \alpha_P, \kappa_A)$. If $||\theta_1||_{p,*} > \alpha_A + \alpha_P,$ then $\sup_{x} h(\theta, (x, a)) - \alpha_A ||x^A_i - x ||_{p} -  \alpha_P||x^P_j - x||_{p}=\infty$. Otherwise, we have
\begin{align*}
    &\sup_{x} h(\theta, (x, a)) - \alpha_A ||x^A_i - x ||_{p} -  \alpha_P||x^P_j - x||_{p}\\
    &\le f\Bigg(\Bigg(\frac{\min(\alpha_A,\alpha_P)||\theta_1||_{*}||x^A_i-x^P_j||}{\alpha_A + \alpha_P} + \frac{\langle\theta_1, \alpha_A x^A_i + \alpha_P x^P_j\rangle}{\alpha_A + \alpha_P}\Bigg) + \langle \theta_2, a\rangle\Bigg) - \min(\alpha_A, \alpha_P) || x^A_i - x^P_j||_{p}.
\end{align*}
\end{restatable}
\begin{proof}
    Because $f$ is a convex function, its biconjugate is itself, so
    \[
    \sup_{b \in [0,1]} -f^*(b) + b \cdot X = f(X). 
    \]
Therefore, we have
\begin{align*}
    &\sup_{x} h(\theta, (x, a)) -\alpha_A||x -  x^A_i||_{p}  - \alpha_P||x-x^P_j||_{p} \\
    &=\sup_{b \in [0,1]} - f^*(b) + (g^i_1 \square g^j_2)(b\theta_1) + \langle b\theta_2, a\rangle\\
    &\le \sup_{b \in [0,1]} - f^*(b) + \left(\frac{b}{\alpha_A + \alpha_P} \right) \left(\min(\alpha_A,\alpha_P)||\theta_1||_{*}||x^A_i-x^P_j|| + \langle\theta_1, \alpha_A x^A_i + \alpha_P x^P_j\rangle\right) \\
    &- \min(\alpha_A, \alpha_P) || x^A_i - x^P_j||_{p} + b \langle \theta_2, a\rangle \\
    &= \sup_{b \in [0,1]} - f^*(b) \\
    &+ b\left(\left(\frac{1}{\alpha_A + \alpha_P} \right) \left(\min(\alpha_A,\alpha_P)||\theta_1||_{*}||x^A_i-x^P_j|| +  \langle\theta_1, \alpha_A x^A_i + \alpha_P x^P_j\rangle\right) + \langle \theta_2, a\rangle\right) - \min(\alpha_A, \alpha_P) || x^A_i - x^P_j||_{p}\\
    &= f\left(\left(\frac{\min(\alpha_A,\alpha_P)||\theta_1||_{*}||x^A_i-x^P_j|| +  \langle\theta, \alpha_A x^A_i + \alpha_P x^P_j\rangle}{\alpha_A + \alpha_P} \right) + \langle \theta_2, a\rangle\right) - \min(\alpha_A, \alpha_P) || x^A_i - x^P_j||_{p}.
\end{align*}
The first inequality follows from Theorem \ref{thm:inf-conv-upper-bound}.
\end{proof}

Now, note that depending on $(a,y)$, $\alpha_A := \frac{\alpha_A}{c(a,y)}$ and $\alpha_P:=\frac{\alpha_P}{c(a,y)}$ in the above lemma and theorem changes because of the equation \eqref{eqn:fairness-c-taken-out}. Therefore, unless $||\theta||_{p,*} \le \min_{(a,y)} \frac{\alpha_A}{c(a,y)} + \frac{\alpha_P}{c(a,y)}$, $\max_{a,y}\sup_{x} \left(c(a, y) \cdot \ell(\theta, (x, a, y))  - \alpha_A d^i_A(x,a) -  \alpha_P d^j_P(x,y)\right)$ evaluates to $\infty$. In other words, we need 
\[
    ||\theta_1||_{p,*} \le \frac{\alpha_A + \alpha_P}{\overline{c}}.
\]

Therefore, via our approximation with $\hat{x}_{i,j}$ as in Section \ref{sec:tractable-optimization}, the optimization problem is 
\begin{align*}
    \min_{\substack{\alpha_A, \alpha_P, \\ \{\beta_{i}\}_{i \in [n_A]}, \{\beta'_j\}_{j \in [n_P]}}} \quad & n_A n_P (\alpha_A r_A + \alpha_P r_P + p_0 \gamma_0 + p_1 \gamma_1) + n_P \sum_{i\in [n_A]} \beta_{i} + n_A \sum_{j \in [n_P]} \beta' \\
    \text{s.t.}\quad& c(a, y) \cdot \ell(\theta, (\hat{x}_{i,j}, a, y)) + \alpha_A \kappa_A|a^A_i - a| +\alpha_P \kappa_P|y^P_j -y |- \min(\alpha_A, \alpha_P)||x^A_{i} - x^P_j||  \\
    &+ \gamma_0\ind[a=0, y=1] + \gamma_1\ind[a=1,y=1]\le \beta_{i} + \beta'_j  \quad i \in [n_A], j \in [n_P], a \in \cA, y \in \{\pm 1\},\\
    &||\theta_1||_{*} \le \frac{\alpha_A + \alpha_P}{\overline{c}}
\end{align*}
As in Section \ref{sec:tractable-optimization}, we can show that the above optimization problem can be solved via two convex optimization problems. And as before, we can show similar additional approximation guarantee. For each fixed $(a,y)$, approximating the $\sup_{x}$ term with $\hat{x}_{i,j}$ will result in approximation error of $2||x^A_i - x^P_j||$ as in Theorem \ref{thm:main-thm}. Therefore, even when we take the max over all $(a,y)$, the overall gap must be bounded by $2c(a,y) ||x^A_i - x^P_j||\le 4||x^A_i - x^P_j||$.
\begin{lemma}
\label{thm:gap-sup-fairness}
Suppose $p \neq 1$ and $p \neq \infty$. If $||\theta_1||_{p,*} \le \alpha_A + \alpha_P$, then
\begin{align*}
    &\max_{a,y}\sup_{x \in \cX} \left(c(a, y) \cdot \ell(\theta, (x, a, y))  - \alpha_A d^i_A(x,a) -  \alpha_P d^j_P(x,y)\right)  \\
    &-\max_{a,y}\left(c(a, y) \cdot \ell(\theta, (\hat{x}_{i,j}, a, y)) + \alpha_A \kappa_A|a^A_i - a| +\alpha_P \kappa_P|y^P_j -y |- \min(\alpha_A, \alpha_P)||x^A_{i} - x^P_j|| \right)\\
    &\le 4 \hat{\alpha} ||x^A_i - x^P_j||.
\end{align*}
\end{lemma}
\begin{proof}
First, fix any $(a,y)$. Using the same argument as in Lemma \ref{lem:holder-lem},
\begin{align*}
    &\left(\sup_{x} \ell(\theta, (x, a, y)) - \frac{\alpha_A}{c(a,y)} ||x^A_i - x ||_{p} -  \frac{\alpha_P}{c(a,y)}||x^P_j - x||_{p}\right) - \left( \ell(\theta, (\hat{x}_{i,j}, a, y)) - \frac{\hat{\alpha}}{c(a,y)} ||x^A_i - x^P_j||\right)\\
    %&\le \left(\ell(\theta, (\bar{x}, a, y)) - \frac{\hat{\alpha}}{c(a,y)} ||x^A_i - x^P_j||\right) - \left( \ell(\theta, (\hat{x}_{i,j}, a, y)) - \frac{\hat{\alpha}}{c(a,y)} ||x^A_i - x^P_j||\right)\\
    &\le 2||x^A_i - x^P_j||
\end{align*}
where $\bar{x}$ is the same as in the proof of Lemma \ref{lem:holder-lem}. Multiplying by $c(a,y)$, we have
\begin{align*}
    &\sup_{x} \left(c(a, y) \cdot \ell(\theta, (x, a, y))  - \alpha_A ||x^A_i - x ||_{p} -  \alpha_P ||x^P_j - x ||_{p}\right) \\
    &- \left(c(a, y) \cdot \ell(\theta, (\hat{x}_{i,j}, a, y))  - \alpha_A ||x^A_i - x ||_{p} -  \alpha_P ||x^P_j - x ||_{p}\right) \\
    &\le 2c(a,y)||x^A_i - x^P_j|| \le 4||x^A_i - x^P_j||.
\end{align*}
because $c(a,y) \le 2$ for any $(a,y)$.

Finally, write 
\[
    (x^*, a^*, y^*) = \arg\max_{(x,a,y)} \left(c(a, y) \cdot \ell(\theta, (x, a, y))  - \alpha_A d^i_A(x,a) -  \alpha_P d^j_P(x,y)\right).
\]

Then, we have
\begin{align*}
    &\left(c(a^*, y^*) \cdot \ell(\theta, (x^*, a^*, y^*))  - \alpha_A d^i_A(x^*,a^*) -  \alpha_P d^j_P(x^*,y^*)\right)  \\
    &-\max_{(a,y)}\left(c(a, y) \cdot \ell(\theta, (\hat{x}_{i,j}, a, y)) + \alpha_A \kappa_A|a^A_i - a| +\alpha_P \kappa_P|y^P_j -y |- \min(\alpha_A, \alpha_P)||x^A_{i} - x^P_j|| \right)\\
    &\le \left(c(a^*, y^*) \cdot \ell(\theta, (x^*, a^*, y^*))  - \alpha_A d^i_A(x^*,a^*) -  \alpha_P d^j_P(x^*,y^*)\right)  \\
    &-\left(c(a^*, y^*) \cdot \ell(\theta, (\hat{x}_{i,j}, a^*, y^*)) + \alpha_A \kappa_A|a^A_i - a^*| +\alpha_P \kappa_P|y^P_j -y^*|- \min(\alpha_A, \alpha_P)||x^A_{i} - x^P_j|| \right)\\
    &=\left(c(a^*, y^*) \cdot \ell(\theta, (x^*, a^*, y^*))  - \alpha_A ||x^A_i - x^* ||_{p} -  \alpha_P ||x^P_j - x^* ||_{p}\right)  \\
    &-\left(c(a^*, y^*) \cdot \ell(\theta, (\hat{x}_{i,j}, a^*, y^*)) -\alpha_A ||x^A_i - x ||_{p} -  \alpha_P ||x^P_j - x ||_{p} \right)\\
    &\le 4||x^A_i - x^P_j||_p
\end{align*}
where the first inequality follows because $-\max_{a,y}$ term cannot be greater than when the inner term is evaluated at $(a^*,y^*)$, and the last inequality follows because for $(a^*,y^*)$, we have shown earlier the gap is at most $4||x^A_i - x^P_j||_p$.
\end{proof}

Furthermore, under the same assumption as in Appendix \ref{sec:opt}, the optimization problem can be even further simplified to 
\begin{align}
\begin{split}
    &\min_{\substack{\alpha_A, \alpha_P, \\ \{\beta_{i}\}_{i \in [n_A]}, \{\beta'_j\}_{j \in [n_P]}}}  n_A n_P (\alpha_A r_A + \alpha_P r_P + p_0 \gamma_0 + p_1 \gamma_1) \\
    &+ \sum_{(i,j) \in M} \max_{a \in \cA, y \in \{\pm 1\}} c(a, y) \cdot \ell(\theta, (\hat{x}_{i,j}, a, y)) + \alpha_A \kappa_A|a^A_i - a| +\alpha_P \kappa_P|y^P_j -y |- \min(\alpha_A, \alpha_P)||x^A_{i} - x^P_j|| \\
    &\text{s.t.}\quad||\theta_1||_{*} \le \frac{\alpha_A + \alpha_P}{\overline{c}}
\end{split}\label{eqn:fairness-prob-final}
\end{align}
where $M$ is defined exactly the same as in Appendix~\ref{sec:opt}.

\section{Missing Details from Section \ref{sec:experiments}}
\label{app:experiments}
All the experiments were performed on one of the authors' personal computer, MacBook Pro 2017, and every experiment took less than an hour. 

We note that as it's standard in practice to output the last iterate instead of the averaged iterate, we use the last iterate of the projected gradient descent instead of the averaged one for all our experiments. Now, the total number of points and the features for each dataset is here along with where the dataset can be found:
\begin{enumerate}
    \item BC (\href{https://archive.ics.uci.edu/ml/datasets/breast+cancer}{https://archive.ics.uci.edu/ml/datasets/breast+cancer}): 569 points with 30 features
    \item IO (\href{https://archive.ics.uci.edu/ml/datasets/ionosphere}{https://archive.ics.uci.edu/ml/datasets/ionosphere}): 351 points with 34 features
    \item HD (\href{https://archive.ics.uci.edu/ml/datasets/Heart+Disease}{https://archive.ics.uci.edu/ml/datasets/Heart+Disease}): 300 points with 13 features
    \item 1vs8 (\href{https://scikit-learn.org/stable/modules/generated/sklearn.datasets.load_digits.html#sklearn.datasets.load_digits}{https://scikit-learn.org/stable/modules/generated/sklearn.datasets.load\_digits.html\#sklearn.datasets.load\_digits}): This is a copy of the test dataset from \href{https://archive.ics.uci.edu/ml/datasets/Optical+Recognition+of+Handwritten+Digits}{https://archive.ics.uci.edu/ml/datasets/Optical+Recognition+of+Handwritten+Digits}). It originally contains 1797 points with 64 points. But after filtering out all the digits except for 1's and 8's, there are 356 points. 
\end{enumerate}
For every dataset, we preprocess the data by standardizing each feature --- that is, removing the mean and scaling to unit variance. 

We take the common feature to be the first 5 features for (BC, HD) and 4 for IO --- i.e. $m_1 = 5$ and $4$ respectively. For 1vs8, we have $m_1=32$, the first half bits of the 8x8 image. And the remaining features are the auxiliary features $\cA$: $m_2=25, 30, 8$, and $32$ for BC, IO, HD, and 1vs8 respectively. For all datasets, we set the test size to be 30\% of the entire dataset. Then, we set $(n_P, v) = (20, 5)$, $(20, 10)$, $(30, 5)$, $(30, 10)$ for BC, IO, HD, 1vs8 respectively. In other words, we imagine the total number of points in our labeled sets $S_P$ and the number of features to be very small. For BC and IO, we also try a case when the number of common features is a lot more (i.e. $m_1 = 25$).

Now we report the best regularization penalties that maximize the accuracy of RLR and RLRO respectively over all experiment runs at the granularity level of $10^{-2}$. The best regularization penalty for RLR and RLRO were $\lambda=(0.07, 0.04)$ for BC ($m_1 = 5$), (0.04, 0.04) for BC ($m_1 = 25$), $(0.02, 0.02)$ for IO ($m_1 =4$), $(0.01, 0.02)$ for IO ($m_1=25$), $(0.08, 0.03)$ for HD, and $(0.08, 0.08)$ for 1vs8. The parameters for data join used for each of the datasets can be found in the table below:
\begin{table}[ht!]
\centering
\begin{tabular}{|l|l|l|l|l|l|l|}
\hline
 & BC ($m_1 = 5$) & BC ($m_1 = 25$) & IO ($m_1 = 4$)& IO ($m_1=25$) & HD & 1vs8\\ \hline
$r_A$  & $0.65$ & $1.65$ & $0.3$ & $1.5$ & $0.65$ &1.85 \\ \hline
$r_P$  & $0.65$ & $1.65$ &$0.3$ & $1.5$ & $0.65$ & 1.85 \\ \hline
$\kappa_A$ & $5$ & $5$ & $10$ & $5$ & $10$ & 5\\ \hline
$\kappa_P$ & $2.5$ & $2.5$ & $5$ & $2.5$ & $5$ & $15$\\ \hline
$k$ & $1$ & $1$ & $1$ & $1$ & $1$ & 1\\
\hline
\end{tabular}
\caption{\label{tab:DJ-parameters} Parameters used for distributionally data join (DJ) for UCI datasets} 
\end{table}

For all of the methods (logistic regression, regularized logistic regression, distributionally robust logistic regression, and our distributionally robust data join), the learning rate used was $7*10^{-2}$ and the total number of iterations was $1500$. 

Finally, we describe how we generated the data that was used to test how well DJ handles distribution shift. First, define
\begin{align*}
    \beta_1 = [1,0,0,0,0,0,0,0,0,0]\quad\text{and}\quad
    \beta_2 = [1,1,1,1,1,1,1,1,1,1].
\end{align*}

For the first group $g=1$, the positive points and negative points were drawn from a multivariate normal distribution with mean $\beta_1$ and $-\beta_1$ respectively both with the standard deviation of 0.2:
\begin{align*}
    x|y=+1,g=1 \sim N(\beta_1, 0.2)\quad\text{and}\quad x|y=-1, g=1 \sim N(-\beta_1, 0.2).
\end{align*}

For the second group $g=2$, the positive points and negative points were drawn from a multivariate normal distribution with mean $\beta_2$ and $-\beta_2$ respectively both with the standard deviation of 0.3:
\begin{align*}
    x|y=+1,g=2 \sim N(\beta_2, 0.2)\quad\text{and}\quad x|y=-1,g=2 \sim N(-\beta_2, 0.2).
\end{align*}

Now, for the first dataset $S_1 = \{(x^1_j, y^1_j)\}_{j=1}^{n_1}$, we had the number of points from group 1 and from group 2 was 400 and 20 respectively. And we had it so that the number of positive and negative points in each group was exactly the same: i.e. 200 positive and negative points for group 1, and 10 positive and 10 negative points for group 2.

For the second dataset, $S_2 = \{(x^2_i, y^2_i)\}_{i=1}^{n_2}$, the number of points from group 1 and from group 2 was 200 and 2000 respectively. The number of positive and negative points in each group was exactly the same once again here. 

Our labeled dataset will be the first two coordinates of the fist dataset, meaning $m_1=2$: 
\[
    S_P = \{(x^1_j[0:2], y^1_j )\}_{j=1}^{n_1}.
\]
Then, we will randomly divide the second dataset so that the 70\% of it will be used as unlabeled dataset $S_A$ and the other 30\% is to be used as the test dataset $S_\text{test}$.
\begin{align*}
    S_A = \{x^2_i\}_{i=1}^{0.7 n_2} \quad\text{and}\quad S_\text{test} = \{(x^2_i, y^2_i)\}_{i=0.7n_2 + 1}^{n_2}.
\end{align*}
Note that $m_2=10$. 

The baselines that we consider for this synthetic data experiment are 
\begin{enumerate}
    \item Logistic regression trained (LR) on $S_P$
    \item Regularized regression trained (RLR) on $S_P$ with $\lambda = 10$
    \item Distributionally logistic regression (DLR) trained on $S_P$  with $r=100, \kappa=10$.
\end{enumerate}

\end{document}